\documentclass[final,12pt]{alt2021} 

\usepackage{microtype}

\usepackage{mathtools, bm}
\usepackage{amssymb, bm}

\usepackage{bbold}
\usepackage{ dsfont}
\usepackage{enumitem}

\usepackage{sectsty}

\DeclarePairedDelimiter\ceil{\lceil}{\rceil}
\DeclarePairedDelimiter\floor{\lfloor}{\rfloor}

\newcommand{\Po}{\mathbb{P}}

\DeclareMathOperator*{\argmax}{arg\,max}

\newcommand{\E}{\mathbb{E}} 
\newcommand{\1}{\mathbb{1}}

\newcommand{\defeq}{\mathrel{\mathop:}=}
\usepackage{xspace}

\newtheorem{assumption}{Assumption}

\makeatletter
\newtheorem{repeatthm@}{Theorem}
\newenvironment{repeatthm}[1]{%
    \def\therepeatthm@{\ref{#1}}
    \repeatthm@
}
{\endrepeatthm@}
\makeatother

\makeatletter
\newtheorem{repeatlem@}{Lemma}

\makeatother

\newcommand{\ubar}[1]{\text{\b{$#1$}}}
\usepackage{comment}

\title[Efficient Algorithms for Stochastic Repeated Second-price Auctions]{Efficient Algorithms for Stochastic Repeated Second-price Auctions}
\usepackage{times}

\altauthor{%
 \Name{Juliette Achddou} \Email{juliette.achdou@gmail.com}\\
 \addr ENS Paris, Université PSL, 1000mercis Group, INRIA
 \AND
 \Name{Olivier Cappé} \Email{olivier.cappe@cnrs.fr}\\
 \addr ENS Paris, Université PSL, CNRS, INRIA%
 \AND
 \Name{Aurélien Garivier} \Email{aurelien.garivier@ens-lyon.fr}\\
 \addr UMPA,CNRS, INRIA, ENS Lyon %
}

\begin{document}

\maketitle


\begin{abstract}
  Developing efficient sequential bidding strategies for repeated
  auctions is an important practical challenge in various marketing
  tasks. In this setting, the bidding agent obtains information, on
  both the value of the item at sale and the behavior of the other
  bidders, only when she wins the auction. Standard bandit theory does
  not apply to this problem due to the presence of action-dependent
  censoring.  In this work, we consider second-price auctions and
  propose novel, efficient UCB-like algorithms for this task. These
  algorithms are analyzed in the stochastic setting, assuming
  regularity of the distribution of the opponents' bids. We
  provide regret upper bounds that quantify the improvement over the
  baseline algorithm proposed in the literature. The improvement is
  particularly significant in cases when the value of the auctioned
  item is low, yielding a spectacular reduction in the order of the worst-case regret. We further provide the first parametric lower bound
  for this problem that applies to generic UCB-like strategies. As
  an alternative, we propose more explainable strategies which are reminiscent 
  of the Explore Then Commit bandit algorithm. We provide a critical analysis of this class of
  strategies, showing both important advantages and limitations. In
  particular, we provide a minimax lower bound and propose a nearly
  minimax-optimal instance of this class.
\end{abstract}

\section{Introduction} \label{intro}
This work is devoted to the design of strategies for bidders participating
in repeated second-price auctions. Second-price auctions have the
unique property of being dominant-strategy incentive-compatible,
which means that a bidder's return in a second-price auction is always
better when she plays truthfully than with any other bidding
strategy. This property
explains the widespread use of second-price auctions in various
sales and marketing contexts. In particular, second-price auctions have long been the
predominant auction structure in the field of programmatic
advertising. There, they allow advertising spaces owned by editors to
be sold to the highest bidder among publishers who need to display
their ads in fractions of a second. It is therefore possible to model
an advertising campaign of one of the publishers as a sequence of
second-price auctions that involve varying sets of other bidders ---since not
all publishers bid on the same inventories--- with an initially unknown
payoff (materialized by clicks or by a purchase amount subsequent
to the display of the advertisement). In this context, the different
publishers typically do not assign the same value to the placements and
thus each bidder has a different private value in the auction. If
reserve prices are used, these can be interpreted as additional
fixed virtual bids that are present in each auction.
In this work, we adopt the point of view of a single bidder
with the aim of developing provably efficient online learning
strategies that maximize her cumulative reward.

\subsection{Model}\label{setting}
We consider that similar items are sold in $T$ sequential second-price
auctions. For $t$ in $1,\ldots,T$, the auction unfolds in the
following way. First, the bidder submits her bid $B_t$ for the item
that is of unknown value $V_t$. The other players submit their bids,
the maximum of which is called $M_t$. If $M_t\leq B_t$ (which includes
the case of ties), the bidder observes and receives $V_t$, and pays
$M_t$. Otherwise, the bidder loses the auction and does not observe
$V_t$.
We make the following additional assumptions. The values $\{V_t\}_{t\geq 1}$ are
independent and identically distributed random variables in the unit
interval [0,1]; their expectation is denoted by $\mathbb{E}(V_t)=v
$. The maximal bids $\{M_t\} _{t\geq 1}$ are independent and
identically distributed random variables in the unit interval $[0,1]$;
their cumulative distribution function is denoted by $F$.

Neither $v$ nor $F$ are assumed to be known initially to the
bidder. The above model corresponds to a continuously-armed bandit
with a very particular structure: the higher the bid, the higher the
probability of observing a value and hence of learning useful
information for estimating the optimal bid $v$ (see
Appendix~\ref{pr:max_utility} for a simple proof that $v$ is actually optimal,under a mild assumption on $F$). Therefore, exploration requires that the
bids $B_t$ are not set too low while exploitation is achieved by
bidding as close as possible to the estimated value of $v$. Under the
optimal bidding policy, i.e., when $B_t=v$, only a fraction $F(v)$ of
the rounds results in an actual observation of $V_t$. In the example
of online advertising, the values are typically
binary with very small expectation (click or conversion rates are
usually less than a few percents). We will therefore be interested in
algorithms that also perform satisfactorily in situations when the
parameter $v$ and  $F(v)$, are very small.
The assumption that the bids $M_t$ are stochastic is well suited to
contexts in which the set of bidders varies in time, the bids
depend on unobservable time-varying contextual information or when the bidders
do not assign the same value to the placements. In this work, we restrict ourselves to
this setting (also considered by \citet{weed2016online} or
\citet{flajolet2017real}) and refer to Section 4 of
\citep{weed2016online} for techniques applicable to other scenarios
where fast learning rates cannot be achieved.

\subsection{Related Works}\label{related_work}
Following the pioneering study by
\cite{vickrey1961counterspeculation} on one-item second-price
auctions, a line of work emerged on mechanism design, a field aimed at
designing mechanisms, and in particular auctions, that satisfy some
fixed properties (for example being dominant strategy incentive
compatible). \cite{myerson1981optimal} proves that second
price auctions with reserve prices, i.e. in which the seller
herself submits a virtual bid, guarantees a maximum return to the
seller when the players are symmetric, while maintaining the dominant
strategy incentive compatibility of the auction.  In particular, the
same author also proves that if the valuations of
each of the bidders are drawn independently from known distributions
then a closed-form expression exists for the optimal reserve
price.


For the case when the
distributions of the valuations of the bidders are
unknown, \cite{blum2004online, cesa2014regret} studied
online learning strategies to maximize the cumulative
payoff of the seller while learning the optimal reserve price in repeated
auctions. \cite{medina2014learning}  instead consider a supervised learning problem with features associated to each of the auctions and adopt a full information setting instead of a bandit setting. They justify this change of setting by the fact that sellers often have access to all the bids. \cite{blum2004online} also work in the full information setting and rely on the weighted majority algorithm \citep{littlestone1989weighted} to learn the reserve price. In the
aforementioned papers, bidders are not assumed to be strategic, which
is fixed by \cite{amin2013learning,mohri2014revenue,drutsa2017horizon,kanoria2017dynamic,vanunts2019optimal,drutsa2020optimal}.

Online learning algorithms
can be used to maximize the cumulative return of the buyers via the
choice of their bid, as was done for the sellers.
\cite{weed2016online} proposed an online algorithm for second-price
auctions in the stochastic and adversarial settings. In particular,
for the stochastic setup, the authors consider a strategy
which consists of bidding the empirical average of the past observed
values plus an exploration bonus. This way
of balancing exploration and exploitation can be framed into the
family of Upper Confidence Bound (UCB) strategies, that were first
used for multi-armed bandits, see \cite{lattimore2018bandit} for an
introduction to multi-armed bandits and \cite{auer2002finite} for the
first analysis of UCB. \cite{weed2016online} provides upper bounds on
the regret of UCBID, under the assumption that $F$ satisfies some form
of ``margin condition'' (see Assumption \ref{ass:margin_cond} in Appendix \ref{app:general_Lemma}). \cite{weed2016online} also exhibit a minimax lower bound that
pertains to the specific case where $M_t$ is
deterministic. \cite{feng2018learning} extend the results of
\cite{weed2016online} to different kind of auctions mechanisms,
including multi-item auctions with stochastic allocation rules.

Another extension is considered by \cite{flajolet2017real}, who study
contextual strategies for repeated second-price auctions with a
budget. Using UCB strategies developed for linear bandits \cite{dani2008stochastic, abbasi2011improved}, \cite{flajolet2017real}
introduce an algorithm for bidding adaptively to contexts that
represent new users or inventories arriving sequentially, under the
assumption that the expected reward is a linear function of the
context. The latter algorithm also includes elements from the Bandit with Knapsacks
framework \cite{badanidiyuru2013bandits}, which are used to handle a financial budget constraint.

\subsection{Main Contributions}\label{main_contrib}

In the stochastic framework considered by \cite{weed2016online} and recalled in Section~\ref{sec:model}, we
propose three new algorithms. We start in Section~\ref{sec:ucb} by considering UCB algorithms that rely on tighter confidence bounds than those used for UCBID. We propose two such algorithms, named kl-UCBID and Bernstein-UCBID, which rely,
respectively, on the use of Chernoff and Bernstein deviation
inequalities. These algorithms are inspired by the works of
\cite{garivier2011kl} and
\cite{audibert2009exploration} for the multi-armed bandit case. Under regularity assumptions, the analysis
reveals an improvement over UCBID that is dramatic in cases in which $v$
is small, yielding a remarkable reduction of the order of the worst-case regret in Theorem~\ref{th:other_strats}.
The analysis requires to fix a gap in the proof scheme of
\cite{weed2016online}, which has an impact for small values of $v$. 
In
section~\ref{sec:lower_bound}, we provide a lower bound valid for optimistic strategies in the
cases when the distribution of $M_t$ admits a bounded density. This
lower bound differs from the result of \cite{weed2016online} because it is
parameter-dependent (rather than minimax) and not restricted to the
case when the distribution of $M_t$ is degenerate (and hence, when
the auction outcome is deterministic given the bid). This new lower
bound highlights the impact of $v$ on the performance of the algorithms.

However UCB strategies can be hard to accept and/or to implement by practitioners in some contexts. With this in mind, we also consider in Section~\ref{sec:ETG} a class of simpler, two-phases strategies 
that bid maximally until a stopping time, then either abandon the bid 
 (typically in
cases when $v$ is deemed to be small) or default to the running average
of observed values. We provide a critical analysis of this approach,
showing both important advantages and limitations: an excellent behavior for relatively a large $v$, and a sub-optimal behavior in case of a small value.  We provide
a minimax lower bound for this class of strategies; it shows that
their worst-case regret is bound to be larger than that of the best UCB algorithm.  
These findings are illustrated by numerical simulations in Section~\ref{sec:simul}.

\section{Model and Assumptions}\label{sec:model}

We precise the setup presented in Section \ref{setting} with the
following notation.  Let
$N_t = \sum_{s=1}^{t} \mathbb{1} \{ M_s \leq B_s\}$ be the number of times
the bidder won up to time $t$, which, we stress, is also the number of
observations of values up to time $t$.  Let
$\bar{V}_t := \frac{1}{N_t}\sum_{s=1}^t V_s \1(M_s\leq B_s)$, be the
empirical mean of the values observed before time $t$ and
$\bar{W}_t := 1/N_t \sum_{s=1}^t (V_s- \bar{V}_t)^2 \1(M_s\leq B_s)$,
the population variance of the values observed before time
$t$. Lastly, $w$ denotes the variance of $V_t$.

\label{regret}
We define the utility of a bidder who submits a
bid $b$ as the difference between the received value and the paid
price: the utility function at time $t$ is therefore expressed as
$U_t(b)=(V_t - M_t) \1 \{b \geq M_t\}$.
As a measure of performance of bidding algorithms, we use the following notion of expected cumulative regret:
\begin{equation*}
R_T \defeq \max_{b \in [0,1]} \sum_{t=1}^T \E [U_t(b)] - \sum_{t=1}^T \E[U_t(B_t)].
\end{equation*}

It is a well known fact that bidding $v$ is an optimal bidding strategy in second price auctions. A proof of this result is given in Appendix~\ref{pr:max_utility}, showing in particular that if the density of $M_t$ vanishes around this
valuation $v$, then $v$ is one of many maximizers of the
utility, while if $F$ admits a
strictly positive density, then $U_t(b)$ is unimodal, in the sense that it is
first non-decreasing and then non-increasing. This is interesting,
because it allows for the use of specific algorithms designed for unimodal
bandits \cite{yu2011unimodal,combes2014unimodal}. We will
see, however, that these algorithms are sub-optimal as they
do not fully exploit the structure of the problem.


The local behavior of $F$ around the unknown value $v$ is a key
parameter that influences the regret achievable by sequential bidding
strategies. In the main body of this paper we focus exclusively on
the most natural assumption that the distribution of the maxima of the
bids of the opponents admits a density that is bounded in an
interval containing $v$.

\begin{assumption}\label{ass:bounded_dens}  \textbf{Locally bounded density.}
There exists $\Delta>0$ such that $F$ admits a density $f$ bounded on $[v,v+\Delta]$, i.e.,  there exists $\beta >0$, such that $\forall x \in [v,v+\Delta],~ f(x)<\beta.$
\end{assumption}

This assumption corresponds to a local version of the margin condition introduced by \cite{weed2016online} (corresponding to their case of $\alpha$=1). Broader results are provided in Appendix, covering the different local behaviors of $F$ around $v$ (as defined by Assumption~\ref{ass:margin_cond}).

\section{UCB-type algorithms}\label{sec:ucb}
\subsection{Specification of the Algorithms}

We start by studying Upper Confidence Bound strategies, similar to
the UCBID algorithm proposed by \cite{weed2016online}. At each time step, the algorithms submit
an upper confidence bound of $v$, thus allowing ---whith high probability--- for at least as
many observations as the optimal strategy (i.e. bidding $v$ from the first round on).
The two proposed algorithms, klUCBID and BernsteinUCBID, differ from UCBID by the use of tighter upper confidence bounds. Namely,
for $t>1$, they respectively submit a bid equal to
$$
B_t = UCB_t(\gamma) = \begin{cases} \min \left(1,\bar{V}_{t-1} + \sqrt{\frac{\gamma \log (t)}{2 N_{t-1}}} \right) \text{ for UCBID }\\
\inf \left\{ x \in (\bar{V}_{t-1},1] : kl( \bar{V}_{t-1},x)=\frac{\gamma \log(t)}{N_{t-1}}\right\}\text{ for klUCBID } \\
\min \left(1, \bar{V}_{t-1} + \sqrt{\frac{2 \bar{W}_{t-1} \log(3t^{\gamma})}{N_{t-1}}} + \frac{3 \log(3 t^\gamma)}{N_{t-1}} \right) \text{ for BernsteinUCBID ,}
\end{cases}
$$
where  $kl(p, q)$ denotes the Kullback Leibler divergence between
two Bernoulli distributions of expectations $p$ and $q$ and where
$\gamma>0$ is a parameter that impacts the exploration level. At time
$t=1$, all three algorithms bid $B_t=1$, ensuring an initial
observation.

\subsection{Upper Bounds on the Regret} \label{sec:upp_bounds}
We start with a new analysis of the UCBID algorithm, which completes that of~\cite{weed2016online} by emphasizing the presence of an important multiplicative factor. 

\begin{theorem}\label{th:UCBID}
  If $F$ satisfies Assumption~\ref{ass:bounded_dens} and $F(v)>0$, then the regret of the UCBID algorithm with
  parameter $\gamma > 1$ is bounded as follows:
$$
R_T \leq   \frac{2\beta\gamma}{F(v)} \log^2 T + O(\log T).
$$
\end{theorem}

This result is proved in Appendix \ref{pr:ucbid_strong} , which also includes the exact (non asymptotic) form of the upper bound. The main
difference with the result of \cite{weed2016online} is the factor
${1}/{F(v)}$ that arises in the upper bound. This factor can be
interpreted as the average time between two successive observations
under the optimal policy, which consists of bidding $v$ and winning
the auction with probability $F(v)$. When $v$ is small, one thus has
to wait a long time before an optimistic bid $B_t$ gets near to $v$,
especially because as $B_t$ tends to $v$, the algorithm is given fewer
and fewer observations to learn from. Furthermore, the lower bound
derived in Section~\ref{sec:lower_bound} also displays the same form
of dependence on $v$. Lemma~\ref{lem:t_to_n} (in Appendix) is the key ingredient
used to control the deviations of the bid $B_t$ from $v$ using a
resampling argument. Contrary to the case of multi-armed bandits,
there is no gap between the optimal policy and other possible strategies. To
obtain the (squared) logarithmic rate of growth of the regret
in Theorem~\ref{th:UCBID}, it is important to recognize that, under
Assumption~\ref{ass:bounded_dens}, the expected utility $\E [U_t(b)]$ is locally
quadratic around $v$.

We now present the results pertaining to the two proposed algorithms
kl-UCBID and Bernstein-UCBID. These bounds show a
significant improvement by including a variance term of primary
importance in applications.

\begin{theorem}\label{th:klUCBID}
  If $F$ satisfies Assumption \ref{ass:bounded_dens}, the kl-UCBID algorithm with parameter
  $\gamma > 1$ yields the following bound on the regret:
$$
R_T  \leq 
8 {\gamma v(1-v)}  \frac{\beta \log(T)^2}{F(v)}  \big(1+o(1)\big) \;.
$$
\end{theorem}

\begin{theorem}\label{th:BernsteinUCBID}
  If $F$ satisfies Assumption \ref{ass:bounded_dens} and $F(v)>0$, the Bernstein-UCBID algorithm
  with parameter $\gamma > 2$ yields a regret bounded as follows :
 \begin{align*}
R_T \leq      \frac{\beta}{F(v)} 8  w \gamma \log^2 (T) + O(\log T).
\end{align*}
\end{theorem}



The proofs of Theorems \ref{th:klUCBID} and \ref{th:BernsteinUCBID}
can be found respectively in Appendix \ref{pr:klucbid_strong} and
\ref{pr:bernsteinucbid_strong}. Theorem  \ref{th:BernsteinUCBID} relies on a non-asymptotic bound that can be found in Appendix \ref{pr:bernsteinucbid_strong}.
The remainder term in Theorem~\ref{th:klUCBID} is less
precise as a consequence of the fact that the kl-UCBID upper
confidence bound does not admit a closed-form expression. Indeed the
term $v(1-v)$ appears as a part of the second order Taylor
approximation of $kl(v',x)$ for values of $v'$ sufficiently close to $v$.
A non-asymptotic bound for kl-UCBID is also proven in Appendix \ref{pr:klucbid_non_asymptotic}, but with larger multiplicative constants.
The leading terms in the bound on the regret of kl-UCBID and the bound for the
regret of Bernstein-UCBID are smaller than that of UCBID, as
$v(1-v)$ is always less than 1/4. The difference becomes particularly
significant for small values of $v$. The comparison between kl-UCBID
and Bernstein-UCBID is more subtle. It is always true that $w$, which
is the variance of the distribution of $V_t$, is smaller than
$v(1-v)$, since $V_t\in[0,1]$; this means that Bernstein-UCBID is a very robust
algorithm. In the case of Bernoulli rewards however, $w=v(1-v)$ and Chernoff's upper confidence bound is tighter than Bernstein's: consistently, we observe in numerical simulations that kl-UCBID dominates
Bernstein-UCBID (see Section~\ref{sec:simul}).
Note that we focus solely on Assumption \ref{ass:bounded_dens} in the main body of the paper for simplicity, but that this assumption  is weaker than Assumption \ref{ass:margin_cond} that is required in the proof of the upper bounds of the regret in the Appendix. Assumption \ref{ass:margin_cond} is more general and extends to all cases when $F$ is continuous in a neighborhood of $v$.  Theorem \ref{th:UCBID_strong} for example gives upper bounds for the regret of UCB depending on the behavior of $F$ around $v$, the worst order of the regret being $O(\log T \sqrt{T})$. On the contrary, if $F$ is discontinuous at $v$, a bound of the order of $\log T \sqrt{T}$ still holds for the regret of all three UCB algorithms, as we will show in Theorem \ref{th:other_strats}. 

\subsection{Lower Bound on the Regret of Optimistic Strategies}\label{sec:lower_bound}
 This section provides a lower bound for the regret of any optimistic algorithm.
 
\begin{theorem} \label{th:lower_bound} 
  We consider all environments where $V_t$ follows a Bernoulli distribution with expectation $v$ and $F$ admits a density $f$ that is bounded from below and above, with $f(b)\geq \ubar{\beta}>0$.
If a strategy is such that, for all such environments,
$R_T\leq O(T^{a})$, for all $a>0$, 
and if there exists $\gamma >0$ such that for all such environments, $\Po(B_t< v)<t^{-\gamma}$,
 then this strategy must satisfy:
\begin{equation}
\liminf_{T \rightarrow \infty} \frac{R_T}{\log T} \geq \ubar{\beta} \frac{v(1-v)}{16 F(v)} \label{eq:lower_bound}.\end{equation}
\end{theorem}

This result is proved in Appendix \ref{pr:lower_bound}. The first
assumption, $R_T\leq O(T^{a})$, is a requirement of minimal uniform
performance similar to the assumption used to prove the lower bound of
\cite{lai1985asymptotically} for multi-armed bandits (see,
e.g., \cite{lattimore2018bandit}). The second
assumption, namely $\Po(B_t< v)<t^{-\gamma}$, implies that the strategy must avoid
underestimating the expected value; the latter assumption is naturally satisfied by the class of
UCB-like strategies.

This result extends the lower bound proved by \cite{weed2016online}
in two important ways. First, it is not restricted to the case where
the maximal opponent bid $M_t$ is constant (which is particular since
there are in this case only two policies that can be optimal by either
losing or winning all auctions). Second, this is a parameter-dependent
bound, rather than a minimax one, that highlights the influence of
the parameters of the problem ($v$ and $F$).

The upper bounds obtained in Section~\ref{sec:upp_bounds} for the
regret of the three algorithms are of the order of $\log^2 T$, when
$F$ satisfies Assumption \ref{ass:bounded_dens}. The lower bound
\refeq{eq:lower_bound} is therefore not sufficient to ensure that the
rate of $\log^2 T$ is actually asymptotically optimal. In contrast, the
dependence on $v$ of the lower bound is enlightening, since it shows that
Bernstein-UCBID and kl-UCBID are close to optimal with respect to $v$,
when $V_t$ is drawn from a Bernoulli distribution. The regret bound
for UCBID however is not on par with the lower bound for smaller
values of $v$, suggesting a degraded performance in this regime (as
will be illustrated by the numerical simulations in
Section~\ref{sec:simul}).

\paragraph{Idea of the proof}
In this paragraph we provide a quick overview of the proof, to highlight its differences from other lower bound arguments. For the complete proof, please refer to Appendix \ref{pr:lower_bound}.

Instead of considering only two alternative models that are difficult to distinguish from each other, and where mistaking one for the other necessarily leads to a high regret, we consider a different alternative for each of the $T$ time steps.
We fix a model in which all $(V_s)_{s=1}^T$ follow a Bernoulli distribution with expectation $v$, and the bids $(M_s)_{s=1}^T$ are distributed according to $F$.
At each time $t$, we consider the alternative model where the values $(V_s)_{s=1}^T$ follow a Bernoulli distribution with expectation $v'_t = v + \sqrt{\frac{v(1-v)}{F(v)t}}$, and the bids $M_t$ are distributed according to $F$.
We apply Le Cam's method to obtain
\begin{equation}
\label{eq:lecam}
\Po_v \left( B_t > \frac{v+ v'_t}{2}\right)+ \Po_{v'_t} \left( B_t < \frac{v+ v'_t}{2}\right)
 \geq  1 - \sqrt{\frac{1}{2} KL(\Po_v^{I_t},\Po_{v'_t}^{I_t})},
 \end{equation}
where we name $I_t$ the information collected up to time $t$ : $(M_{t-1}, V'_{t-1}, \ldots M_1, V'_1)$.

The KL divergence between $\Po_v^{I_t}$ and $\Po_{v'_t}^{I_t}$ can be proved to be equal to $kl(v, v'_t)\E_v[N_t]$. Furthermore, Lemma \ref{lem:limit_Nt} in Appendix \ref{pr:lower_bound} shows that the expected ratio of won auctions tends to $F(v)$ under the assumptions of Theorem \ref{th:lower_bound}. Using this lemma, we obtain that $\forall \epsilon>0, \exists t_1(\epsilon), \forall t\geq t_1(\epsilon) $, 
\begin{equation}\label{eq:asymp_KL_LB} KL(\Po_v^{I_t},\Po_{v'_t}^{I_t}) \leq kl(v,v'_t)(1+\epsilon)F(v)t.
\end{equation}
Combining Equations \ref{eq:lecam} and \ref{eq:asymp_KL_LB} yields an asymptotic bound on $\Po_v \left( B_t > \frac{v+ v'_t}{2}\right)+ \Po_{v'_t} \left( B_t < \frac{v+ v'_t}{2}\right).$
The second assumption of Theorem \ref{th:lower_bound}
 ensures that the probability of underbidding under the alternative model can be bounded by $t^{-\gamma}$
, so that it translates to a bound on the probability of overbidding by more than $\sqrt{\frac{v(1-v)}{F(v)t}}$ under the fixed model.
 Combining this bound and the fact that $\E_v[(B_t-v)^2]\geq  (v-\frac{v+ v'_t}{2})^2\Po_v \left( B_t > \frac{v+ v'_t}{2}\right)$ yields an asymptotic bound on 
$ \sum_{t=1}^T \E_v[(B_t-v)^2]$ that reads
\begin{align*}
\liminf_{T \rightarrow \infty} \frac{\sum_{t=1}^T \E_v[(B_t-v)^2]}{\log T} \geq \frac{v(1-v)}{8 F(v)}.
\end{align*}
If $F$ admits a density bounded from below by $\ubar{\beta}$,  Lemma \ref{lem:density_f} (Appendix  \ref{pr:density_f}) shows  that  
$R_T(v) \geq \frac{\ubar{\beta}}{2} \sum_{t= 1}^T \E_v[(B_t-v)^2]$:  this concludes the proof.

\subsection{Analysis of Worst Case Regrets}
In addition to the parametric lower bound result, we provide another
element that suggests that UCBID is outperformed by the other two
strategies for small values of $v$. We do so by studying the maximal regret $\max_{v\in[0,1]} R_T(v)$ of
the three strategies for all values of $v\in[0,1]$. This maximal regret is likely to be reached when $v$ is of the order of $T^{-1/2}$ (respectively $T^{-1/3}$) for UCBID (respectively BernsteinUCBID) (see Appendix
\ref{pr:other_strats} for more details).

\begin{theorem}\label{th:other_strats}
Without further assumption, the maximal regrets of UCBID, BernsteinUCBID and klUCBID are 
$O(\sqrt{T}\log T)$. If $F$ has a density that is bounded from below and above by positive constants, the maximal regret of UCBID remains of the same order, while it is reduced to $O(T^{\frac{1}{3}}\log^2 T)$ for BernsteinUCBID and to $O(\log^2 T)$ for klUCBID.

\end{theorem} 

This result suggests that there exist important differences between
UCBID, BernsteinUCBID and above all klUCBID. The proof highlights the overwhelming superiority of the latter in the presence of low values.

The worst-case upper-bounds should in principle be compared with a minimax lower bound. We  refer to the lower bound of \cite{weed2016online}, which is of the order of $\log T$. By comparing it with the worst case bounds contained in this section and assuming that the obtained worst-case upper-bounds are tight, we observe that the maximal regret of UCBID is in reality not on par with the logarithmic lower-bound of \cite{weed2016online}, since it is of the order of $\sqrt{T}\log T$, and not $\log T$.
Only klUCBID is guaranteed to close the gap between the worst-case regret and this minimax lower bound.
The difference in the order of the worst-case regrets comes from the fact that the optimistic bonus contains a term linked to the \emph{variance} in the case of klUCBID, whereas UCBID uses the very crude upper-bound $1/4$.
This has dramatic consequences for small values of $v$, causing a much higher worst-case regret for UCBID than for klUCBID. 
 
Note that the bound on the regret of
UCBID does not require any assumption on $F$ and
can even be proved without the assumption that the maximal bids $M_t$
of the adversaries are iid. (see the proof in Appendix
\ref{pr:other_strats}). Whereas in the multi-armed bandit setting, the known minimax lower bound for stochastic regret matches that of adversarial regret, this is not the case for second price auctions.
Namely, \cite{weed2016online} prove that the worst case adversarial regret is lower bounded by $\frac{1}{32}\sqrt{T}$, while we prove here that the worst case regret of klUCBID is upper-bounded by $O(\log^2T)$.
This discrepancy is in particular due to the fact that in the adversarial setting the maximal bids of the opponents $M_t$ can be chosen arbitrarily, so that they do not necessarily satisfy Assumption \ref{ass:bounded_dens}.
Therefore, in the adversarial setting, the regret is linear around the optimal bid rather than quadratic, leading to higher regrets.

\paragraph{Idea of the proof.}
We only provide a sketch of the proof, please see Appendix \ref{pr:other_strats} for details.
We start with UCBID, with a different approach than in Section \ref{sec:upp_bounds}. 
The regret $R_T$ of any strategy is upper-bounded by
$\sum_{t=1}^T \E[(M_t-v) \1\{v\leq M_t \leq B_t\}] + \sum_{t=1}^T \Po(B_t<v),$
as proved in Lemma \ref{lem:ineq_regret} in Appendix \ref{sec:gen_lemmas_ucb}.
$R_T$ is hence also bounded by $\sum_{t=1}^T \E[(B_t-v) \1\{M_t \leq B_t\}] + \sum_{t=1}^T \Po(B_t<v),$ that is, the sum of the overbidding margins and the probabilities of underbidding.
It is easy to show that the second term is smaller than a constant and can therefore be neglected for a large enough value of $\gamma$. This leaves us with the task of bounding the overbidding margins.
By definition, $B_t-v = \bar{V}_t-v+ \sqrt{\frac{\gamma \log T}{N_t}}$. 
The deviations to the left, $\bar{V}_t-v$, can be bounded with high probability by $\sqrt{\frac{\gamma \log T}{N_t}}$, which results in the high probability bound: $B_t-v  \leq 2 \sqrt{\frac{\gamma \log T}{N_t}}$.
Since $N_t$ is incremented each time the auction is won, that is, each time $M_t$ is smaller than $ B_t$,
 we obtain
 $$\sum_{t=1}^T 2 \sqrt{\frac{\gamma \log T}{N_t}} 
 \1\{M_t \leq B_t\} \leq \sum_{n=1}^T 2 \sqrt{\frac{\gamma \log T}{n}} \leq 2 \sqrt{\gamma T \log T}.
 $$
This yields the order of the worst case regret shown in Theorem \ref{th:other_strats} for UCBID.

As for BernsteinUCBID, the worst-case bound is also not a direct consequence of the bound presented in Section \ref{sec:upp_bounds}.
We can use the same argument as for UCBID, but by bounding  the deviations to the left using Berstein's instead of Hoeffding's inequality.
This yields 
$$\E_v[R_T]\leq 2 \sqrt{2  T w\log(3  T^{\gamma})} +O_T,$$
where  $O_T = O(\log^2(T))$ does not depend on $v$. Using this bound for small values of $v$ and the bound used in Section \ref{sec:upp_bounds} for large values of $v$ yields the given bound for BernsteinUCBID.

The proof of the worst-case regret for klUCBID requires some more work, as it cannot be based on the bound presented in Section \ref{sec:upp_bounds} because of its asymptotic nature. It is a consequence of the time-dependent bound derived in Lemma \ref{lem: klucbid_th_non_asymp} in Appendix \ref{pr:klucbid_non_asymptotic}. This bound reads  $$\E[R_T]\leq C \frac{v}{F(v)} \log^2(T) + O_T,$$ where C is a constant and  $O_T$ does not depend on $v$ and $O_T = o(\log^2 T)$ . 
 We refer to Appendix \ref{pr:klucbid_non_asymptotic} for details.

\section{ETG algorithms}\label{sec:ETG}

In multi-armed bandits, Explore Then Commit (ETC) strategies are
a simplistic alternative to UCB. They are composed of two distinct
phases: exploration, designed to estimate the return of each arm,
and exploitation, where the agent plays the arm deemed more
profitable. In the two arm case, ETC with an adaptive choice of the
length of the exploration phase can reach competitive but sub-optimal performance (see
e.g. \citet{garivier2016explore}). In the bidding model, due to the
continuous nature of the problem at hand, the ETC strategy is not
viable since any approach that simply commits to a fixed bid after some time is
bound to incur a large regret.

We consider instead the following class of strategies, termed Explore Then Greedy (ETG). 
In the exploration phase, the maximal value of the bid ($B_t=1$) is chosen to force observation. 
After a well-chosen stopping time, the bidder chooses either to abandon the bids (choosing $B_t=0$), or to continue with the running average of observed values (greedy
phase). 
In the sequel, the exploration phase ends as soon as one is reasonably certain  that the value $v$ will either never more be under-estimated by a factor larger than $2$ (stopping time $\tau_1$), or is not worth bidding (stopping time $\tau_0$).

ETG strategies have the advantage of being simple both to explain and
to implement. Indeed, they only require the possibility of executing tests
and of computing the running average. In the context of digital advertising, simplicity is critical, as most bidders operate through specialized platforms that only
allow simple operations because of the very high frequency of auctions. ETG
strategies are also closer to the methods that are naturally
implemented by marketers and are therefore easily explainable.

\subsection{ETGstop}
We propose one instance of ETG, that we call ETGstop, defined by the following choice of stopping times $\tau_1$ and $\tau_0$: \\
\begin{equation}\label{def:tau_1}\tau_1 := \inf\left\{t \in [1,T]: \exp\left(-\frac{t L_t}{8}\right)\leq \frac{1}{T^2}\right\}, ~ \tau_0 =  \inf\left\{t \in [1,T]: U_t \leq \frac{1}{T^{\frac{1}{3}}}\right\}
\end{equation}
where we denote by
$L_t= \min\{ v\in [0, \bar{V}_t[: \exp\left(- t kl(\bar{V}_t,v)\right) \leq
{1}/{T^2}\}$ and by
$U_t= \max\{ v\in [\bar{V}_t,1[: \exp(- t kl(\bar{V}_t,v)) \geq
{1}/{T^2}\}$ the kl-lower and upper confidence bound for the
confidence level ${1}/{T^2}$.

This choice of $\tau_1$ allows to guarantee that with high
probability, if $\tau_1$ is smaller than $\tau_0$, all bids will be
larger than $\frac{v}{2}$ in the second phase. Indeed, we prove that
for all $n$,
$\Po(\bar{V}(n)\leq \frac{v}{2})\leq \exp(-n kl(v, \frac{v}{2})) \leq \exp(-\frac{nv}{8})$, where
$\bar{V}(n)$ denotes the empirical mean of the first $n$ observed
values (see Lemma \ref{lem:kl} in Appendix \ref{app:general_Lemma}). With high probability,
$\exp(-\frac{nv}{8})\leq \exp(-\frac{nL_{\tau}}{8})\leq
\frac{1}{T^2},$ for all $n> \tau$, since $L_t$ is a lower confidence
bound of $v$. Therefore, the probability that there exists a time
step in the second phase for which the average of the observed values is
less than $v/2$ is small. Choosing this stopping time as the
starting point of the greedy phase therefore ensures a minimal ratio
of won auctions in this second phase.

\begin{theorem}\label{th:ETGstop}
 If $F$ admits a density $f$, that satisfies 
$\exists ~ \ubar{\beta}, \beta>0, \forall x\in [0,1],~ ~ \ubar{\beta} \leq f(x) \leq \beta,$ then the regret of ETGstop satisfies :
$$\max_{v \in [0,1]} R_T(v)\leq O(T^{\frac{1}{3}}\log^2 T),$$
$ \text{and if }v> \frac{1}{T^{\frac{1}{3}}}$,$$\text{ then}~~ R_T(v)\leq 7 + \frac{64\log(T) + 60T^{-1/2}}{v}+ \frac{4}{F(v/2)} + \beta \frac{ \log^2 T}{F(v/2)}.$$ 
\end{theorem}

This result (proved in Appendix \ref{pr:ETGstop}) shows that the order of the regret of ETGstop
is similar to that of UCBID, klUCBID and
BernsteinUCBID, when $v$ is large enough. It is also worth noticing
that the bound of the worst case regret for ETGstop is similar to that
of BernsteinUCBID and compares favorably to that of UCBID. In the numerical experiments, we show that ETGstop can outperform all other algorithms for large values, whereas its regret for small values $v$ is larger than those of UCBID and (of course) klUCBID. Can smarter ETG strategies be designed, which would reach a significantly better performance?
The following lower bound answers negatively.

\subsection{Minimax Lower Bound of the regret for ETG strategies}\label{sec:lower_bound_etg}
\begin{theorem}\label{th:lower_bound_ETG}
If F admits a density lower-bounded by $\ubar{\beta}>0$, then the regret of any ETG strategy satisfies
\begin{equation} \max_{v\in[0,1]}R_T(v) \geq \frac{\ubar{\beta}}{4} \left(T^{\frac{1}{3}}-1 \right). \label{eq:low_bound_ETG}
\end{equation}
\end{theorem} 

The proof of this result is given in Appendix \ref{pr:minimax_lower}. This
bound makes ETGstop minimax optimal in the class of ETG strategies, up
to a $\log^2 T$ multiplicative factor. According to Theorem \ref{th:other_strats}, the worst-case regret of klUCBID is of the order of $\log^2 T$. Thus, the lower bound in  \eqref{eq:low_bound_ETG} proves that ETG strategies are bound to be suboptimal in the minimax sense.
It should not come as a surprise: in the MAB framework already, Explore Then Commit strategies are sub-optimal by a factor $2$ on every two-armed Gaussian problem \citep{garivier2016explore}, and it is natural to think that ETG strategies should suffer from the same drawback. However, the difference between UCB and ETG strategies for second price auctions is way more substantial, since the regret of ETG strategies is necessarilyof an order much larger than that of klUCBID.

\section{Simulations}\label{sec:simul}
In Figure \ref{fig:_evolution_in_time1}, we plot the average regret of
UCBID, kl-UCBID and Bernstein-UCBID, when $V_t$ is drawn from a
Bernoulli distribution of expectation 0.2 and $M_t$ is drawn from a
uniform distribution. The regret is computed on 10,000 time steps and
averaged on 50,000 Monte Carlo trials. The plot clearly shows that on the first
steps and for this particular configuration, kl-UCBID outperforms
UCBID and Bernstein-UCBID. Bernstein-UCBID has a larger regret than UCBID and
kl-UCBID. This comes from the fact that the variance is $v(1-v)$ in
the present case, which makes kl-UCBID a better candidate.

In Figure \ref{fig:_evolution_in_time2}, we plot the average regret of
UCBID, kl-UCBID and Bernstein-UCBID, when $V_t$ takes only two values
0.195 and 0.205, each with probability 1/2 and $M_t$ is uniform. Here, the horizon is 100,000 and we make 5,000 Monte Carlo trials. As
predicted by the analysis of Section~\ref{sec:upp_bounds}, Bernstein-UCBID outperforms UCBID on the long
run. However, the regret of both UCBID and kl-UCBID is smaller than
that of Bernstein-UCBID on small times. This comes from the fact that
the optimistic bonus is dominated, in the beginning, by the term
${3 \log(3 t^\gamma)}/{N_t}$  which tends to be
larger than the optimistic bonuses of the two other algorithms during
the first time steps. Note that this is an extreme case, in which the
greedy strategy would perform well, since the variance is very small.

We now compare the UCB algorithms to the simpler algorithm
ETGstop and to other benchmarks. Other candidate algorithms
include the greedy strategy, which submits the current average of the
observed values, the $K$-armed bandit UCB on a discrete set of values
in the $(0,1)$ interval and a unimodal bandits algorithm like LSE
\citep{yu2011unimodal}. Note that UCB on discrete values does not take
the structure of the problem into account. These algorithms have been
tested for Bernoulli distributed values of mean $0.3$ and uniform
$M_t$ and the result is plotted in Figure
\ref{fig:_evolution_in_time_all}. The algorithm named GreedyBID is the
greedy one and the one named UCB\_discrete is UCB on discrete
values. The case displayed in Figure \ref{fig:_evolution_in_time_all}
corresponds to UCB run on a uniform grid of 100 values. It was checked
that it is not possible to significantly improve the performance of
UCB by varying the discretization: making it coarser increases the
regret, while making it finer slows down
learning. Figure~\ref{fig:_evolution_in_time_all} shows that the
greedy strategy, the $K$-armed bandit UCB on a discrete set of values
and LSE behave linearly and are therefore rapidly
outperformed by the UCB type algorithms. We ran the simulations with a modified version of ETGstop that awaits  $\tau'_1 = \inf\{t \in [1,T]: t L_t \geq 2\log T\}$ instead of $\tau_1$
before it starts the greedy phase.
This new stopping time guarantees that the bids are larger then $v/10$ ---rather than $v/2$ for the analyzed version of ETGStop--- with high  probability.
 Indeed, for all $n>\tau'_1$, $\Po(\bar{V}(n)\leq \frac{v}{10})\leq\exp(-nkl(v, v/10))\leq \exp(-n v)\leq \frac{1}{T^2}$ (the third inequality is obtained thanks to Lemma \ref{lem:kl} in Appendix \ref{app:general_Lemma}) since  $L_{\tau'_1}<v$ with high probability.
 Therefore, this version of ETGStop waits less than the analyzed version before it starts the second phase, and is thus more competitive in the regime where $v$ is small, but is likely to accumulate more regret in the greedy phase.
 This version of ETGstop does not perform as well as UCBID and klUCBID, but outperforms BernsteinUCBID on the long run, showing that it can perform satisfactorily for large enough values.

In Figure \ref{fig:_evolution_in_mean}, we plot the regret after 5,000
time steps of UCBID, kl-UCBID, Bernstein-UCBID and ETGStop when $M_t$ is
uniform and $V_t$ follows a Bernoulli distribution, as a function of
the expectation of $V_t$. We ran the simulations on 20 different values of $v$, using 50,000 Monte Carlo trials in each case.
As expected from the result of Section~\ref{sec:ETG}, ETGstop performs poorly when $v$ is small but outperforms the other strategies when $v$ is large enough. Note that the peak of the regret of ETGStop is obtained around $v= T^{-\frac{1}{3}}$ as suggested by the proof of Theorem \ref{th:ETGstop}.
 The regret
of UCBID seems to increase when $v$ gets close to $T^{-\frac{1}{2}}$, contrasting with the
regrets of BernsteinUCBID and kl-UCBID which seem to reach a maximum
for larger values, as suggested by the proof of Theorem \ref{th:other_strats}.

\begin{figure}[ht]
\floatconts
{fig:exps}
{\caption{Numerical simulations}}
{
\subfigure[Regret plots of three UCB algorithms for values $V_t \sim \text{Ber}(0.2)$ and uniform $M_t$.]{%
\centering
{\includegraphics[width=0.43\linewidth]{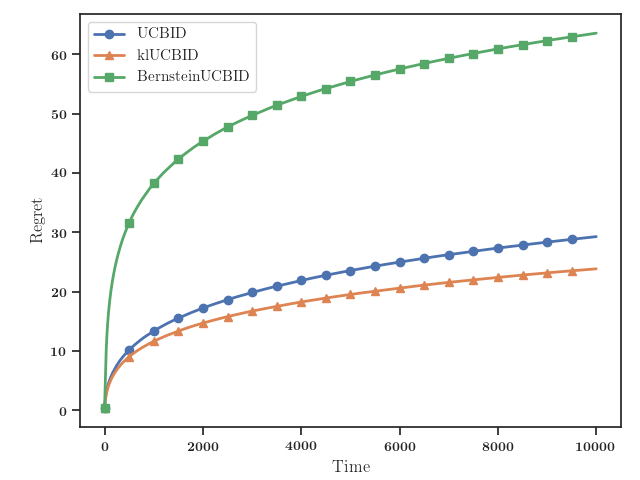}}
\label{fig:_evolution_in_time1}
}\quad
\subfigure[Regret plots of three UCB algorithms for $V_t$ supported on $\{0.195,0.205\}$ and uniform $M_t$]{%
\centering
{\includegraphics[width=0.43\linewidth]{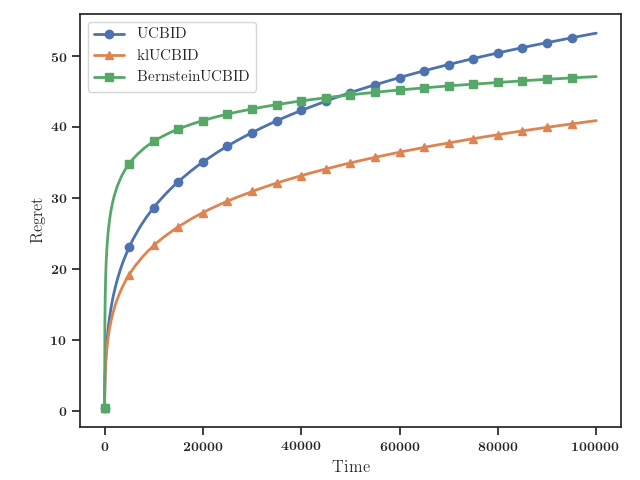}}
\label{fig:_evolution_in_time2}
}
\subfigure[Comparison with ETGstop and other algorithms, for $V_t\sim \text{Ber}(0.3)$ and uniform $M_t$.]{%
\centering
{\includegraphics[width=0.43\linewidth]{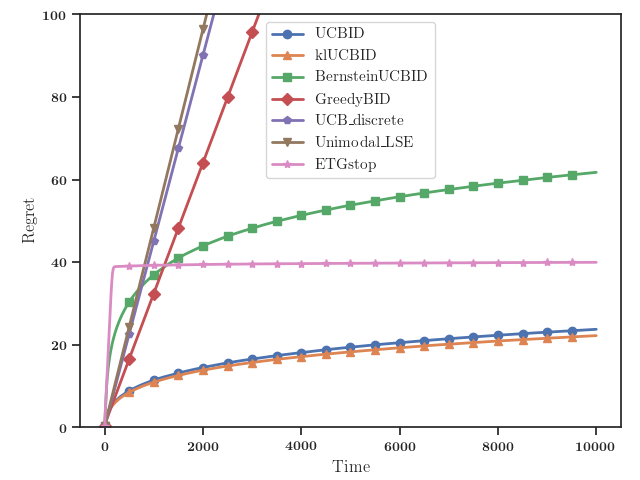}}
\label{fig:_evolution_in_time_all}
}\quad
\subfigure[Regret at time 5000 of studied policies for uniform $M_t$ and Bernoulli-distributed $V_t$ of varying mean $v$.]{%
\centering
{\includegraphics[width=0.43\linewidth]{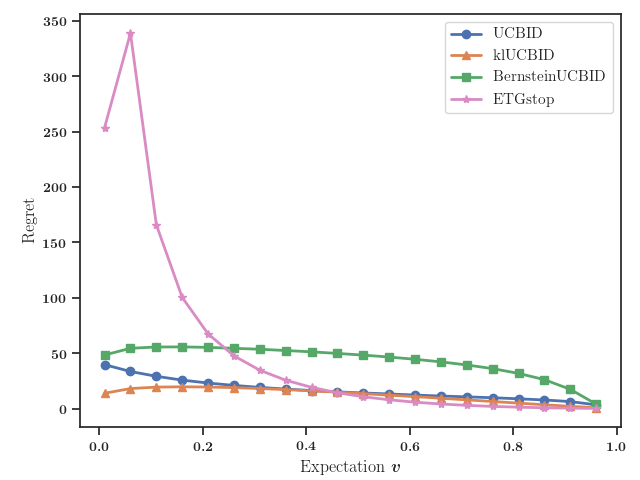}}
\label{fig:_evolution_in_mean}
}
}
\end{figure}

\vspace{-10 mm}

\section{Conclusion}
We have proposed new UCB strategies for buyers in repeated second-price
auctions. We have shown that they improve significantly over the state of the art, both
experimentally and theoretically. Our analysis emphasizes the role played by
two important factors: the frequency of wins
at the optimal bid and the variance of the value distribution that takes a dramatic
importance when the value of the item is close to $0$ or $1$.
We have also shown that simpler strategies 
may reach asymptotic performances similar to those of  UCB-like algorithms. 
In the online advertising context, such strategies could be set up more easily by smaller stakeholders who do not necessarily own a programmatic bidder, but bid through demand-side platforms.
 However, by studying the worst-case regret of such strategies, we have shown
that this simplicity has a price: they perform poorly when the value $v$ lies under a certain threshold.

\clearpage

\bibliography{ref}

\clearpage
\appendix
{\Large\textbf{Supplementary Material}}\vspace*{1em}\\
\paragraph{Outline.} The appendix is structured as follows: Appendix \ref{app:general_Lemma} contains general lemmas used in the rest of the appendix. Appendix \ref{app:pr_UCB_BernsteinUCB} contains the proof of the parametric upper bounds of the regret of UCBID and BernsteinUCBID. A brief outline of the proof scheme precedes the proofs (Appendix \ref{app:sketch}).  The asymptotic and non asymptotic upper-bound for the regret of klUCBID can be found in Appendix \ref{app:klUCBID}. The proof of the parametric lower bound for standard UCB strategies follows in Appendix \ref{pr:lower_bound} . Appendix \ref{pr:other_strats} deals with the study of the worst case regret of UCB strategies. Appendix \ref{pr:ETGstop} provides the proof of the upper bound of the regret of ETGstop, while Appendix \ref{pr:minimax_lower} contains that of the minimax lower bound of ETG strategies.

\paragraph{Additional Notation and Assumptions.} The sequel requires the use of additional notations.
\begin{itemize}
\item We define $(V(n))_{n\geq 1}$ by $V(n) \defeq V_{\tau_n}$
where $\tau_n=\inf\{t:N_t = n\}$. $V(n)$ is the $n-th$ observation of
a value.
\item $\bar{V}(n):=1/n\sum_{s=1}^n V(s)$ is the mean of the $n$ first
observed values.
\item $\bar{W}(n) := 1/n \sum_{s=1}^n
(V(s)- \bar{V}(s))^2$ is the population variance of the $n$ first
observed values.
\item We set  $V'_s = V_s$ if $ M_s\leq B_s,$ and $V'_s = \emptyset$
otherwise.
\item Let $\mathcal{F}_t = \sigma((M_s,V'_s)_{s\leq t})$ be the
  $\sigma$-algebra generated by the the bid maxima and the values observed up
  to time $t$. 
\item We denote by $L(n)= \min\{ v\in [0, \bar{V}(n)[: \exp(- t kl(\bar{V}(n),v)) \leq \frac{1}{T^2}\}$  and by $U(n)= \max\{ v\in [\bar{V}(n),1[: \exp(- n kl(\bar{V}(n),v)) \leq \frac{1}{T^2}\}$ 
\item We also define the instantaneous regret by
$$r_t \defeq U_t(v) - U_t(B_t),$$ in such a way that
$R_T = \sum_{t=1}^T \E[r_t]$.  
\end{itemize}

Whereas the results presented in the main body of the paper require the assumption that $F$ admits a locally bounded density, we prove broader results, under the following assumption.

\begin{assumption}\label{ass:margin_cond} \textbf{Local margin condition.}
F satisfies the local margin assumption, with parameter $\alpha>0$, if there exists strictly positive constants $\Delta$ and $\beta$ such that 
$$ F(x)-F(v) \leq \beta (x-v)^{\alpha}, ~~~~~~ \forall x \in [v, v+\Delta]\;. $$
\end{assumption}

Assumption \ref{ass:bounded_dens} is a special case of Assumption \ref{ass:margin_cond} with $\alpha = 1$.

Alternatively, we will use the uniform version of the margin condition.
\begin{assumption}\label{ass:margin_cond_uniform} \textbf{Uniform margin condition.}
F satisfies the uniform margin assumption, with parameter $\alpha>0$, if there exists a positive constant $\beta$ such that 
$$ F(x)-F(v) \leq \beta (x-v)^{\alpha}, ~~~~~~ \forall x \in [v, 1]\;. $$
\end{assumption}

\section{General Lemmas}\label{app:general_Lemma}
\subsection{Maximal Utility}
\label{pr:max_utility}
 \begin{lemma}\label{lem:max_utility}
Assume that $F$ admits a density. The utility writes $U_t(b)=(V_t - M_t) \1 \{b \geq M_t\}$.
If there exists $a < v$ (respectively $b> v$), such that $F$ is constant on $[a,v]$  (respectively $[v,b]$) then $[a, v] \subset \argmax_{b\in[0,1]} U_t(b) $ (respectively $[v,b] \subset \argmax_{b\in[0,1]} U_t(b) $).
Otherwise, $$\argmax_{b\in[0,1]} {\E\big[U_t(b)\big]} = \{v\} .$$
 In either case,
 $\max_{b \in [0,1]} \sum_{t=1}^T \E [U_t(b)] = \sum_{t=1}^T \E[\{v-M_t\} \mathbb{1}\{M_t<v\}].$
 \end{lemma}
 \begin{proof}
If $F$ admits a density $f$, then $F$ is continuous and non decreasing.  We can write:
\begin{align*}
 \E[U_t(b)] &= \int_0^b (v-m) f(m) dm \\
 &= \int_0^b v f(m) dm - \int_0^b m f(m) dm\\
 &= v F(b) - [  m F(m) ]^{b}_0 + \int_0^b F(m) dm\\
 &= (v-b) F(b) + \int_0^b F(m) dm,
 \end{align*}
 
 \begin{align*}
 \E[U_t(v)]  - \E[U_t(b)] &=  \int_0^v F(m) dm-(v-b) F(b) - \int_0^b F(m) dm \\
 &= \int_b^v (F(m) - F(b))dm .
 \end{align*}
This latter quantity is non negative. It vanishes if and only if $F$ is constant, on $[b,v]$, if $b<v$, and on $[v,b]$, if $v\leq b$.
 In particular,  
 \begin{align*}
\E[U_t(b)] 
&\leq \E([(V_t-M_t)\1(M_t\leq v)]),
 \end{align*}
 and $\max_{b \in [0,1]} \sum_{t=1}^T \E [U_t(b)] = \sum_{t=1}^T \E[\{v-M_t\} \mathbb{1}\{ M_t\leq v\}]$.\\
 
  Equivalently, we can prove the same close form of $\max_{b \in [0,1]} \sum_{t=1}^T \E [U_t(b)]$ by observing that:
$$
 \max_{b \in [0,1]} \sum_{t=1}^T \E[(V_t - M_t)\mathbb{1}\{ M_t\leq b\}] \\
 = \max_{b \in [0,1]} \sum_{t=1}^T \E[(v - M_t)\mathbb{1}\{ M_t\leq b\}] ,
$$
because  $ V_t$  does not depend on  $M_t$.\\
As $(v- M_t)\1\{M_t\leq b\} \leq (v-M_t)\1\{M_t\leq v\} $ , one has
 \begin{align*}\max_{b \in [0,1]} \sum_{t=1}^T \E[U_t(b)] &= \sum_{t=1}^T \E[(v - M_t)\mathbb{1}\{M_t\leq v\}]\\&= \sum_{t=1}^T \E[(V_t - M_t)\mathbb{1}\{M_t \leq v\}].
 \end{align*}

\end{proof}

\subsection{An Expression of the Regret}
The instantaneous regret can be
rewritten as
\begin{equation}\label{eq:instantaneous_regret}
r_t = (M_t-V_t) \1\{v < M_t \leq B_t\}
 +(V_t- M_t)\1\{B_t < M_t \leq v\}.
\end{equation}
 In fact, the regret may only be different from zero when $M_t$ lies
 between $v$ and $B_t$, otherwise bidding $v$ and $B_t$ lead to the
 same utility. In the right hand side of \eqref{eq:instantaneous_regret}, the first
 (respectively second) term describes the case when the bidder wins
 (respectively loses) the auction.

\subsection{Quadratic lower and upper bounds of the regret}\label{pr:density_f}

\begin{lemma}\label{lem:density_f}
If $F$ admits a density $f$, which satisfies 
$$\exists ~ \ubar{\beta}, \beta>0, \forall x\in [0,1],~ ~ \ubar{\beta} \leq f(x) \leq \beta;$$ Then,
$$ \frac{\ubar{\beta}}{2} \sum_{t=1}^T \E[(B_t-v)^2] \leq R_T \leq \frac{\beta}{2} \sum_{t=1}^T \E[(B_t-v)^2] .$$
\end{lemma}

\begin{proof}

The instantaneous regret writes :
\begin{align*}
\E[r_t(b)] &= \E[U_t(v)] - \E[U_t(b)] \\
 &= \int_b^v F(m) dm - (v-b)F(b)\\
 &= \int_b^v (F(m)- F(b)) dm\\
 &= \int_b^v \int_{b}^m f(u) du dm
\end{align*}
where the first equality comes from Lemma $\ref{lem:max_utility}$.

Hence 
\begin{equation*}\begin{cases} \ubar{\beta} \int_b^v (m-b) dm \leq r_t(b) \leq \beta \int_b^v (m-b) dm \text{ if }b\leq v, \\
\ubar{\beta} \int_v^b (b-m) dm \leq r_t(b) \leq \beta \int_v^b (b-m) dm  \text{ if }v < b,
\end{cases}
\end{equation*}

Thus 
$$\frac{1}{2} \ubar{\beta} (b-v)^2 \leq r_t(b) \leq \frac{1}{2} \beta  (b-v)^2 .$$

\end{proof}
\subsection{Bounds on the KL divergence between Bernoulli distributions}

\begin{lemma}\label{lem:gen_Pinsker_strong}
$\forall p \in [0,1], q \in ]0,1]$
$$kl(p,q)\geq  \frac{(p-q)^2}{2\tilde{x} (1 - \tilde{x})}\geq \frac{(p-q)^2}{2\tilde{x}},$$
where $\tilde{x}:=\frac{p+2q}{3}$
\end{lemma}
\begin{proof}
Thanks to Taylor's form, 
$kl(p,q) = \frac{(p-q)^2}{2}\int_0^1 \psi''(q+ s(p-q))2(1-s) ds$, where $\psi:x \mapsto kl(x,q)$ and hence $\psi''(x) = \frac{1}{x(1-x)}.$
We apply Jensen's inequality, using $\int_0^1 2(1 - s)ds =\frac{1}{3}$.
\begin{align*}\int_0^1 \psi''(q+ s(p-q))2(1-s) ds & \geq  \psi''\left(\int_0^1 q+ s(p-q)2(1-s)ds\right)\\
& \geq \psi'' \left(\frac{1}{3} (p-q)+ q\right)\\
&\geq \psi'' \left(\frac{1}{3} (p+ 2q)\right).
\end{align*}
\end{proof}

\begin{lemma}\label{lem:kl}
 For any $v$ in $[0,1]$, for any $\alpha\in]-1,\frac{1}{v}-1[$, $$\mathrm{kl}\big((1+\alpha) v, v\big) \geq \frac{\alpha^2v}{2\big(1+\alpha/3\big)}.$$
 \end{lemma}
 \begin{proof}
This lemma is a direct consequence of Lemma  \ref{lem:gen_Pinsker_strong}.
\end{proof}

\subsection{An Inequality Based on Benett's Inequality}
\begin{lemma}\label{lem:benett}
For all $n$, such that $n\leq \frac{\log T}{v}$, 
$$  \Po\left(\bar{V}(n) -v\geq v + \frac{\log T}{n}\right)\leq \frac{1}{T}.$$
For all $n$, such that $n\geq \frac{\log T}{v}$, 
$$ \Po\left(\bar{V}(n)-v\geq \sqrt{2v \frac{\log T}{n}}\right)\leq \frac{1}{T}.$$
\end{lemma}
\begin{proof}
We use Benett's inequality: 
$\Po(\bar{V}(n) -v\geq x)\leq \exp\left(- n w g(\frac{x}{w})\right)$
where $g(u)= (1+u)\log(1+u)-u$. Let $h_x(w) = wg(x/w)$. For all $x>0$, $h_x$ is decreasing, its derivative being $h_x'(w)= \log(1+\frac{x}{w})- \frac{x}{w}$. Therefore $$\Po(\bar{V}(n) -v\geq x)\leq \exp\left(- n v(1-v)g\left(\frac{x}{v(1-v)}\right)\right) \leq \exp\left(- n vg \left(\frac{x}{v}\right)\right).$$

We call $v^+(n)$ the solution of: $v g(\frac{x}{v})= \frac{\log T}{n}$ in $[0, \infty[$ (it will become clear that the solution is unique). 
$v^+(n)$ satisfies $\Po(\bar{V}(n) \geq v^+(n))\leq \frac{1}{T}$.
We aim at bounding $v^+(n)$.

Since $g '(u)=\log(1+u)$, and $g''(u)= \frac{1}{1+u}>0$ when $u>-1$, $g$ is strictly convex on $[-1,\infty]$. Also, we can define an inverse $g^{-1}$ on $]0 ,\infty]$.

Additionally, $g(e-1)=1$ and $g'(e-1)=1$, so $g(x)\geq x - (e-2)\geq x- 1$, $\forall x<e-1$, by convexity. When $x\geq e-1$, we can prove $g(x) \geq \left(\frac{x}{e-1}\right)^2$ by analyzing the function $g(x) - \left(\frac{x}{e-1}\right)^2$.

Therefore if $\frac{\log T}{vn} \geq 1$, $$v^+(n)=  v g^{-1}\left(\frac{\log T}{vn}\right)\leq v + \frac{\log T}{n},$$ 
and if $\frac{\log T}{vn} \leq 1$,
 $$v^+(n)= v g^{-1}\left(\frac{\log T}{vn}\right)\leq  (e-1)\sqrt{v  \frac{\log T}{n}} \leq \sqrt{ 2v \frac{\log T}{n}}.$$ 
\end{proof}
\section{Proof of Theorem \ref{th:UCBID} and Theorem \ref{th:BernsteinUCBID}}\label{app:pr_UCB_BernsteinUCB}

\subsection{Sketch of the Proofs}\label{app:sketch}
The methods used to obtain the bounds follow the same pattern for the
three UCB algorithms. In this sketch of proof, we focus for simplicity on the case when Assumption \ref{ass:margin_cond_uniform} is valid, see Sections \ref{pr:ucbid_strong} and \ref{pr:bernsteinucbid_strong} for more
general arguments.

We first observe that the regret $R_T$ of any strategy is upper-bounded by
\begin{equation}\sum_{t=1}^T \E[(M_t-v) \1\{v\leq M_t \leq B_t\}] + \sum_{t=1}^T \Po(B_t<v).\label{eq:ineq_regret} 
\end{equation}
as proved in Lemma \ref{lem:ineq_regret} in Section \ref{sec:gen_lemmas_ucb} below.
We start by bounding the second term of the right hand side of inequality \eqref{eq:ineq_regret} by bounding the deviations to the left of $\bar{V}_t$.
Furthermore, when $F$ satisfies Assumption \ref{ass:margin_cond_uniform} with parameter
$\alpha$ and constant $\beta$, we can prove that
$$
\E[(M_t-v) \1\{v\leq M_t \leq B_t\}] 
\leq \E \left[ \frac{\beta\left( B_t -v\right)_+^{1+ \alpha}}{F(v)}\1\{M_t<B_t\}\right]\;,
$$
by conditioning on the intersection of $\mathcal{F}_{t-1}$ and the
$\sigma$-algebra generated by $\1\{M_t<B_t\}$. 
This is proved in Lemma \ref{lem:bound_margin_cond} in Section \ref{sec:gen_lemmas_ucb}. Then we use the
following result, proved by a re-sampling argument. (See Section \ref{sec:gen_lemmas_ucb} for the proof.)

\begin{lemma} \label{lem:t_to_n}   
\begin{equation}
\E\left[\sum_{t=1}^T \frac{\beta (B_t^{} - v)_+^{1+\alpha}}{F(v)} \1\{M_t \leq B_t \}\right]
 \leq \frac{\beta}{F(v)} \E\left[ \sum_{n=1}^{T}(B^{+}(n) - v)_+^{1+ \alpha} \right],\label{eq:lem_t_to_n}
\end{equation}
where \\
$B^+(n) = \min \Big( 1,  \bar{V}(n) + \sqrt{\frac{\gamma \log T}{2 n}}\Big)$ for UCBID, \\
$B^+(n)  =  \{ x :x> \bar{V}(n), kl( \bar{V}(n),x)=\frac{\gamma \log T}{n}\}$  for kl-UCBID,\\
$B^+(n)  = \min \Big(1, \bar{V}(n) + \sqrt{\frac{2 \bar{W}(n) \log(3T^{\gamma})}{n}} + \frac{3 \log(3 T^\gamma)}{n}\Big) $ for Bernstein-UCBID.
\end{lemma}

The virtual bids $B^+(n)$ (indexed by the number $n$ of observations
rather than by the time $t$) are obtained by replacing the value of
$t$ by the horizon $T$ in the expression of the bids. They are a clear
upper bound on $B(n)$. Although the horizon is not known a priori, it is convenient to consider these virtual bids in the analysis. 
Bounding the right hand side sum in Lemma \ref{lem:t_to_n} is not hard in the
case of UCBID and Bernstein-UCBID, using the following
inequality for each term: for any non negative sequence $(A_n)$,
\begin{align*}
\E[X_+^{1+ \alpha} ] &=  \E [X_+^{1+ \alpha} \1\{X\leq A_n\}]+\E [X_+^{1+ \alpha} \1\{A_n < X\}]\\
    & \leq \E[A_n^{1+\alpha}] + \Po(A_n < X)\;,
\end{align*}
where $X = (B^{+}(n) - v)$, knowing that $X\in[0,1]$. Setting $A_n$ equal  to the minimum of 1 and two times the optimistic bonus (e.g. $A_n =\min(1,2\sqrt{{\gamma\log T}/({2n})})$ in the case of UCBID), one sees that the probability  $\Po(A_n < X)$ is smaller than the probability that the mean value be larger than $v$ plus the optimistic bonus. We bound this latter probability, by bounding the deviations to the right of $\bar{V}_t$.
Using this inequality term-wise leads to the upper bounds on the
regret presented in the paper for both UCBID and Bernstein-UCBID.

For kl-UCBID, the argument is somewhat different because there is no closed-form expression for the optimistic bonus, and the decomposition above cannot be
used directly. Another difficulty is the fact that the Kullback-Leibler function
for Bernoulli distribution is not symmetric. The proof of the upper
bound on the regret of kl-UCBID is therefore more technical, but
follows the same general pattern.

 \label{sec:upper_bounds_UCBID_and_Bernstein}

\subsection{Proof of the General Lemmas used for the Upper Bound of the regret of UCB strategies}\label{sec:gen_lemmas_ucb}
\subsection{General Bound on the Regret}
\begin{lemma}\label{lem:ineq_regret}
The regret $R_T$ of any strategy is upper-bounded by

\begin{equation}\sum_{t=1}^T \E[(M_t-v) \1\{v\leq M_t \leq B_t\}] + \sum_{t=1}^T \Po(B_t<v).
\end{equation}
\end{lemma}
\label{pr:ineq_regret}

It holds 
$$
\E[r_t |M_t, B_t]= (M_t-v) \1\{v < M_t \leq B_t\}
 +(v- M_t)\1\{B_t < M_t \leq v\}.
$$
 
Lemma \ref{lem:ineq_regret} follows from \eqref{eq:instantaneous_regret} and uses the fact that
$\E[(v- M_t)\1\{B_t<M_t<v\}] \leq \E[\1\{B_t<M_t\}]$, since all the
variables at stake take their values in $[0,1]$.

\subsection{Bound due to the Margin Condition}

\begin{lemma}\label{lem:bound_margin_cond}
If F satisfies Assumption (\ref{ass:margin_cond}) with parameter $\alpha$ and constant $\beta$, 
the expected cumulative regret is bounded by:
\begin{align}
R_T \leq &\sum_{t=1}^T \E\left[\frac{\beta (B_t - v)_+^{\alpha+1}}{F(v)} \1\{M_t < B_t \}\right] +\sum_{t=1}^T  \Po(B_t<v).\label{eq:margin_cond}
\end{align}
\end{lemma}
Compared to the bound in Equation (\ref{eq:ineq_regret}), only the first term has changed. 
The new one is slightly easier to bound, because instead of having to control the unknown quantity $M_t-v$, we only have to control $B_t-v$. 
\begin{proof}

Let $\mathcal{G}_t$ denote the sigma algebra generated by  the intersections of elements of $ \mathcal{F}_{t-1}$ and $\sigma(\1\{M_t \leq B_t \})$. We consider the following conditional expectation:
$$
\E\left[(M_t-v) \1\{0 \leq M_t - v \leq B_t - v \} \left| \mathcal{G}_t  \right.\right]\\ 
 =  \frac{\int_{v}^{B_t} (x-v) dF(x)}{F(B_t)} \1\{M_t \leq B_t \} \1\{v \leq B_t\}
.$$
Thanks to the  margin condition,
if $v \leq B_t$, 
$$\int_v^{B_t}(x-v)dF(x)\leq (B_t-v) \left(F(B_t) - F(v) \right)\leq \beta(B_t-v)^{1+\alpha},
$$
and $ F(B_t)\geq F(v) $ when $v \leq B_t$.
Therefore, 
$$
\E\left[(M_t-v) \1\{0 \leq M_t - v \leq B_t - v \} \left| \mathcal{G}_t  \right.\right]\\ 
\leq  \frac{\beta(B_t-v)^{1+\alpha}}{F(v)} \1\{M_t \leq B_t \} \1\{v \leq B_t\}.
$$

We stress that the above bound may be improved in some cases: indeed, if for example $M_t$ is distributed uniformly, then the exact value of the integral is $\frac{1}{2}(B_t-v)^{2}$, which is half of our upper bound (since the margin condition parameter is 1 in this case).\\

Now it only remains to apply the chain rule and to sum over T terms to obtain the following bound :
$$
 \sum_{t=1}^T \E\left[(M_t-v) \1\{0 \leq M_t - v \leq B_t - v \}\right]\\s
\leq \sum_{t=1}^T \frac{\beta(B_t-v)_+^{1+\alpha}}{F(v)} \1\{M_t \leq B_t \}.
$$

\subsection{Proof of Lemma \ref{lem:t_to_n}}
We prove Lemma \ref{lem:t_to_n} by using a re-sampling argument.
\begin{align*}
\E[\sum_{t=1}^T (B_t - v)_+^{\alpha+1} \1\{M_t \leq B_t \}]
 & \leq \E\left[ \sum_{t=1}^T \sum_{n=1}^{t}(B_t - v)_+^{1+ \alpha} \1\{N_t = n, N_{t-1}= n-1\} \right] \\
& \leq  \E\left[ \sum_{t=1}^T \sum_{n=1}^{t}(B^+(n) - v)_+^{1+ \alpha} \1\{N_t = n, N_{t-1}= n-1\} \right] \\
& \leq  \E\left[ \sum_{n=1}^{T}\sum_{t=1}^T(B^+(n) - v)_+^{1+ \alpha} \1\{N_t = n, N_{t-1}= n-1\} \right]\\
& \leq \E\left[ \sum_{n=1}^{T}(B^+(n) - v)_+^{1+ \alpha} \right].
\end{align*}
\end{proof}
The second inequality follows from a simple re-writing, where we slightly abuse notation by taking $N_0=0$.
The second inequality uses the fact that $B^+(n)$ is an upper bound of $B(n)$.
The third one follows from an inversion of the sums: note that we did not lose any information, since when $t>n$, $\1\{N_t = n, N_{t-1}= n-1\}=0.$

\subsection{ Upper Bound on the Regret of UCBID under Assumption \ref{ass:margin_cond_uniform}}
\label{pr:ucbid_strong}
In this paragraph, we prove the following version of Theorem \ref{th:UCBID}, under a margin condition holding on $[v,1]$. For a proof of the localized version, see section  \ref{sec:localized}.

\begin{theorem}\label{th:UCBID_strong} 
 If  $F$ satisfies Assumption \ref{ass:margin_cond} with parameter $\alpha$ around $v$ on $[v, 1]$ and $F(v)>0$, then the UCBID algorithm  with parameter $\gamma > 1$ yields a regret  bounded as follows :  
$$ R_T \leq \begin{cases}
 C_{\gamma} + \frac{\beta}{F(v)} \left(\frac{2 ( \gamma \log T)^{\frac{1+\alpha }{2}}}{1- \alpha} \left((T)^{\frac{1-\alpha}{2}} + 1\right)+1\right)\\
  \text{ if  } \alpha <1,\\
  C_{\gamma} + \frac{\beta}{F(v)} (( 2 \gamma \log T) (\log T + 1)+1) \\
   \text{ if } \alpha =1,\\
C_{\gamma} +  \frac{\beta}{F(v)} \left[ \frac{6\gamma \log T}{\alpha-1} +1\right] \\
\text{ if } \alpha >1.
\end{cases}$$
 where $C_{\gamma} \defeq  \sum_{t=1}^{\infty} e \sqrt{\gamma} \frac{\log t +1}{t^{\gamma}}>0$.
\end{theorem}

We first need the following lemma.

\begin{lemma} \label{lem:scndterm} The UCBID strategy ensures that the sum of probabilities in Lemma \ref{lem:bound_margin_cond} is bounded by a constant:
\begin{align*}\sum_{t=1}^T \Po(v > B_t) \leq 
 & C_{\gamma} 
 \defeq \sum_{t=1}^T e  \frac{\sqrt{\gamma} \log t +1}{t^{\gamma}}.
\end{align*}
\end{lemma}

\begin{proof}
We use Lemma 10 in \cite{cappe2013kullback}, using the fact that $V_t$  is a random variable bounded in [0,1] and thus $1/2$ subgaussian.
\begin{align*}\Po(v > B_t) &= \Po(v - \bar{V}(N_t) >  \sqrt{\frac{\gamma \log t}{2 N_t}} )\\
& = \Po(\sqrt{2N_t}(v- \bar{V}(N_t)) > \sqrt{{\gamma \log t} }) \\
&\leq \Po( \exists n<t,  \sqrt{2n}(v-\bar{V}(n) > \sqrt{{\gamma \log t} }) 
\\&\leq e \lceil \sqrt{\gamma \log t \log t } \rceil \exp{(-\gamma \log t)}\\
& \leq e  \frac{\sqrt{\gamma} \log t +1}{t^{\gamma}}.
\end{align*}
\end{proof}
Set $X=(B^+(n) - v)$.
As explained in Section \ref{app:sketch}, we use the following  inequality :  
\begin{align*}
\forall A_n>0, \\
    \E[X_+^{1+ \alpha} ] &=  \E [X_+^{1+ \alpha} \1\{X\leq A_n\}]+\E [X_+^{1+ \alpha} \1\{A_n < X\}]\\
    & \leq \E[A_n^{1+\alpha}] + \Po(A_n < X)
\end{align*}

Here we take $A_n$ equal to exactly the minimum of 1 and 2 times the optimistic bonus $A_n = \min \big( 1,2 \sqrt{ \frac{ \gamma \log T}{2n}} \big),$
so that $\Po(A_n < X)$ is smaller than the probability that the mean overestimates $v$ by more than the optimistic bonus. 
Hence, 
$$
 \E[(B^+(n) - v)_+^{1+ \alpha} ]  \\
     \leq \left(A_n\right)^{1+\alpha}  +\Po\Big(2\sqrt{\frac{\gamma \log T}{2n}}< B^+(n)-v\Big).
$$
It holds that
 \begin{align*}   
  \Po \left( \min \left( 1,2\sqrt{\frac{\gamma\log T}{2n}}\right) < B^+(n)-v \right)
    & \leq \Po\left( 2\sqrt{\frac{\gamma \log T}{2n}}< B^+(n)-v\right) + \Po(1  < B^+(n)-v )\\
    &\leq  \Po\left(2\sqrt{\frac{\gamma \log T}{2n}}< \bar{V}(n) - v+ \sqrt{\frac{\gamma\log T}{2n}}\right)\\
    & \leq \Po\left(\sqrt{\frac{\gamma \log T}{2n}}< \bar{V}(n) - v\right)    \\
    &  \leq  T^{-\gamma},
\end{align*}where the last line follows from Hoeffding's inequality. 

By combining Lemma \ref{lem:bound_margin_cond}, Lemma \ref{lem:t_to_n} and Lemma \ref{lem:scndterm}, we finally obtain
\begin{align*}
R_T&\leq C_{\gamma} + \sum_{n=1}^{T}\beta/F(v) \E[(B^+(n)-v)_+^{1+\alpha}]
\\ & \leq C_{\gamma} + \beta/F(v) \left( \sum_{n=1}^{T} A_n^{1+\alpha} + T^{-\gamma}\right)
\\ & \leq C_{\gamma} + \beta/F(v) \left( \sum_{n=1}^{T} \left({\frac{2 \gamma\log T}{n}}\right)^{\frac{1+\alpha}{2}} + 1\right)
\\& \leq C_{\gamma} + \beta/F(v) \left(   \left({2 \gamma \log T}\right)^{\frac{1+\alpha}{2}} \sum_{n=1}^{T}\left(\frac{1}{n} \right)^{\frac{1+\alpha}{2}}+1  \right),
\end{align*}
since $A_n \leq \frac{2 \gamma\log T}{n}, ~ \forall n \in \mathbb{N}$.\\
For  $\alpha >1,$ we improve the bound by using that $A_n = 1,~ \forall n \leq 2 \gamma \log T. $ This yields, 
$$
R_T \leq C_{\gamma} + \beta/F(v) \Bigg(  T^{1-\gamma} +  2 \gamma \log T +  \Big({2 \gamma \log T}\Big)^{\frac{1+\alpha}{2}} \sum_{n= 2 \gamma \log T}^{T}\Big(\frac{1}{n}\Big)^{\frac{1+\alpha}{2}}\Bigg),
$$
We conclude by observing that :
\begin{align*}
\sum_{n=1}^{T}\Big(\frac{1}{n}\Big)^{\frac{1+\alpha}{2}} \leq 
\begin{cases}
  \frac{2}{1- \alpha} (T)^{\frac{1-\alpha}{2}} +1 \text{ if  } \alpha <1\\
   \log T  +1 \text{ if } \alpha =1, 
\end{cases}
\end{align*}
and that
\begin{equation*}
\sum_{n= 2 \gamma \log T}^{T}\Big(\frac{1}{n}\Big)^{\frac{1+\alpha}{2}} \leq 2 \frac{\left(2 \gamma \log T\right)^{\frac{\alpha-1}{2} }}{\alpha -1} \text{ if } \alpha > 1.
\end{equation*}

 \subsection{ Upper Bound on the Regret of Bernstein-UCBID under Assumption \ref{ass:margin_cond_uniform}}\label{pr:bernsteinucbid_strong}
We  prove the following version of Theorem \ref{th:BernsteinUCBID}, under a margin condition holding on $[v,1]$. The proof under a localized margin condition is given in Appendix \ref{sec:localized}.
\begin{theorem}\label{th:BernsteinUCBID_strong}
 If  $F$ satisfies Assumption \ref{ass:margin_cond} with parameter $\alpha$ around $v$ on $[v, 1]$  and $F(v)>0$, the Bernstein-UCBID algorithm  with parameter $\gamma > 2$ yields a regret  bounded as follows : 
\begin{align*}
R_T \leq & C_{\gamma}' +
 \begin{cases}
  \frac{\beta}{F(v)}  (8 w \log (3 T ^{\gamma}))^{\frac{1+\alpha}{2}}  \frac{2}{1- \alpha} \left(T^{\frac{1-\alpha}{2}}+1\right) \\
  +\frac{\beta}{F(v)} \Big((6 c_1 \log (3 T ^{\gamma}))^{{1+\alpha}} +1\Big)
   \text{ if  } \alpha <1,\\
   \frac{\beta}{F(v)} 8  w \log (3 T ^{\gamma}) (\log T+1)\\
   +\frac{\beta}{F(v)} \Big((6 c_1 \log (3 T ^{\gamma}))^{{1+\alpha}} +1\Big) \text{ if } \alpha = 1,\\
      \frac{\beta}{F(v)} (\frac{8w(\alpha +1)}{\alpha-1}+\frac{1}{\alpha})  \log (3 T ^{\gamma})  \text{ if } \alpha > 1,
\end{cases}
\end{align*}
where $c_1 \defeq \sum_{n=1}^{\infty} \frac{1}{n^{{1+ \alpha}}} \leq \sum_{n=1}^{\infty} \frac{1}{n^{2}}= \pi^2 /6$, and   $C_{\gamma}' \defeq \sum_{t=1}^{T} t^{1-\gamma}.$
\end{theorem}

We first need the following lemma.

\begin{lemma} \label{lem:scndterm-Bernstein} The Bernstein-UCBID strategy ensures that the first term in Equation \eqref{eq:margin_cond} in Lemma \ref{lem:bound_margin_cond} is bounded by a constant:
$$\sum_{t=1}^T \Po(v > B_t) \leq 
 \sum_{t=1}^{T} t^{1- \gamma}  \defeq C_{\gamma}'.
$$
\end{lemma}
\begin{proof}
\begin{align*}
 \sum_{t=1}^T \Po \left( B_t < v\right)& \leq  \sum_{t=1}^{T} \Po \left( \bar{V}_t + \sqrt{\frac{2 \bar{W}_t \log(3t^{\gamma})}{N_{t}}} +   \frac{3 \log(3 t^\gamma)}{N_t}< v \right) \\
& \leq  \sum_{t=1}^{T} \sum_{n=1}^{t} \Po \Bigg( v > \bar{V}(n) + \sqrt{\frac{2 \bar{W}(n) \log(3t^{\gamma})}{n}}  + \frac{3 \log(3 t^\gamma)}{n} \Bigg) \\
& \leq  \sum_{t=1}^{T} \sum_{n=1}^{t} t^{-\gamma} \\
& \leq  \sum_{t=1}^{T} t^{1-\gamma},
\end{align*}
where we use a simple union bound for the second inequality, and the Bernstein's deviation inequality for the third inequality. 
\end{proof}
Let $X$ denote $B^+(n) - v$.
As explained in Section \ref{app:sketch}, we use the following  inequality :  
\begin{align*}
\forall A_n>0, \\
    \E[X_+^{1+ \alpha} ] &=  \E [X_+^{1+ \alpha} \1\{X\leq A_n\}]+\E [X_+^{1+ \alpha} \1\{A_n < X\}]\\
    & \leq \E[A_n^{1+\alpha}] + \Po(A_n < X)
\end{align*}
Here we choose $A_n$ as 2 times the optimistic bonus 
$$A_n = \min \left(1,2 \left(\sqrt{\frac{2 \bar{W}(n) \log(3T^{\gamma})}{n}} +   \frac{3 \log(3 T^\gamma)}{n}\right)\right).$$
It holds
\begin{align}
&\E[(B^+(n) - v)_+^{1+ \alpha} ] \notag \\ 
    & \leq  \E \left[\min\Bigg(1,  \Bigg(\sqrt{\frac{8 \bar{W}(n) \log(3T^{\gamma})}{n}} +   \frac{6 \log(3 T^\gamma)}{n}\Bigg)^{1+\alpha} \Bigg)\right] \label{eq_Bern_1} \\    
&~~~~ +\Po\Bigg(\sqrt{\frac{8 \bar{W}(n) \log(3T^{\gamma})}{n}} +   \frac{6 \log(3 T^\gamma)}{n}<B^+(n) - v\Bigg)  \notag
    \\
    & \leq \E \Bigg[\min \Bigg(1,\left(2 \sqrt{\frac{2 \bar{W}(n) \log(3T^{\gamma})}{n}} \right)^{1+\alpha}  + \left(\frac{6 \log(3 T^\gamma)}{n}\right)^{1+\alpha}\Big)\Bigg] \label{eq_Bern_} \\
    &~~~~ + \Po\Bigg(\sqrt{\frac{8 \bar{W}(n) \log(3T^{\gamma})}{n}} +   \frac{6 \log(3 T^\gamma)}{n} <   
    \bar{V}(n) - v  + \sqrt{\frac{2 \bar{W}(n) \log(3T^{\gamma})}{n}} +   \frac{3 \log(3 T^\gamma)}{n}\Bigg)  \notag
    \\
    & \leq \E \left[ \min\Big(1, \Big(\frac{8 \bar{W}(n) \log(3T^{\gamma})}{n}\Big)^{\frac{1+\alpha}{2}}\Big)\right] +  \min \Bigg(1,\left(\frac{6 \log(3 T^\gamma)}{n}\right)^{1+\alpha}\Bigg) \label{eq_Bern_2} \\
   &~~~~ + \Po\Bigg(\sqrt{\frac{2 \bar{W}(n) \log(3T^{\gamma})}{n}} +   \frac{3 \log(3 T^\gamma)}{n} <\bar{V}(n) - v \Bigg) \notag   \\
    & \leq  \E \left[\bar{W}(n) \right]^{\frac{1+\alpha}{2}} \min \Bigg( 1,  \Big(\frac{8 \log(3T^{\gamma})}{n}\Big)^{\frac{1+\alpha}{2}} \Bigg) +\min\Bigg(1, \left(\frac{6 \log(3 T^\gamma)}{n}\right)^{1+\alpha}\Bigg)+ T^{-\gamma} \notag,
\end{align}
 when $\alpha \leq 1$. 
 The last inequality stems from Jensen's inequality, which can only be applied  when $\alpha \leq 1$, and from Bernstein's deviation inequality.\\
 When $\alpha \geq 1,$ we can bound $\E[(B^+(n) - v)_+^{1+ \alpha} ] $ by $$ \E \left[\bar{W}(n) \right] \min \Big(1,\Big(\frac{8 \log(3T^{\gamma})}{n}\Big)^{\frac{1+\alpha}{2}} \Big) +  \min \Big(1\left(\frac{6 \log(3 T^\gamma)}{n}\right)^{1+\alpha}\Big) + T^{-\gamma}.$$
Let us come back to the case when  $\alpha \leq 1$.
Since $\E \left[\bar{W}(n) \right] = \frac{n-1}{n} w \leq w$ and $\gamma >1$,
$$
\E[(B^+(n) - v)_+^{1+ \alpha} ]  
     \leq  w^{\frac{1+\alpha}{2}} \Big( \frac{ 8 \log(3T^{\gamma})}{n}\Big)^{\frac{1+\alpha}{2}} +\left(\frac{6 \log(3 T^\gamma)}{n}\right)^{1+\alpha} + 1,
$$

By combining Lemma \ref{lem:bound_margin_cond}, Lemma \ref{lem:t_to_n} and Lemma \ref{lem:scndterm}, we obtain
\begin{align*}
R_T&\leq C_{\gamma}' + \sum_{n=1}^{T}\beta/F(v) \E[(B^+(n)-v)_+^{1+\alpha}]
\\& \leq C_{\gamma}' + \beta/F(v) \Bigg( \sum_{n=1}^{T}  w^{\frac{1+\alpha}{2}} \Big( \frac{ 8 \log(3T^{\gamma})}{n}\Bigg)^{\frac{1+\alpha}{2}}  + c_1 (6 \log (3 T ^{\gamma}))^{1+\alpha} +1 \Bigg),
\end{align*}

where $c_1 \defeq \sum_{n=1}^{\infty} \frac{1}{n^{{1+ \alpha}}} \leq \sum_{n=1}^{\infty} \frac{1}{n^{2}}= \pi^2 /6$.

When $\alpha>1,$ we argue that we can bound $\E[(B^+(n) - v)_+^{1+ \alpha} ] $ by $$ \E \left[\bar{W}(n) \right] \min \Big(1,\Big(\frac{8 \log(3T^{\gamma})}{n}\Big)^{\frac{1+\alpha}{2}} \Big) + \min \left(1, \left(\frac{6 \log(3 T^\gamma)}{n}\right)^{1+\alpha} \right)+ T^{-\gamma},$$ since the first three inequalities \eqref{eq_Bern_1},\eqref{eq_Bern_}, \eqref{eq_Bern_2},  still hold. 
The regret is therefore bounded by 
\begin{align*}
R_T&\leq C_{\gamma}' + \sum_{n=1}^{T}\frac{\beta}{F(v)} \E[(B^+(n)-v)_+^{1+\alpha}]
\\
& \leq C_{\gamma}' + \frac{\beta}{F(v)} \Bigg( \sum_{n= 8\log(3 T^{\gamma})}^{T}   w \left( \frac{ 8 \log(3T^{\gamma})}{n}\right)^{\frac{1+\alpha}{2}}   
\\&~ +\sum_{n= 8\log(3 T^{\gamma})}^{T}  {\frac{(6 \log (3 T ^{\gamma}))}{n^{1+ \alpha}}}^{{1+\alpha}} + 8 \log(3 T^{\gamma})+1\Bigg)
\end{align*}
because $\Big(\frac{8 \log(3T^{\gamma})}{n}\Big)^{\frac{1+\alpha}{2}} \leq 1$  and $\left(\frac{6 \log(3 T^\gamma)}{n}\right)^{1+\alpha}\leq 1,$ when $n\geq 8\log(3 T^{\gamma})$.

We conclude by observing that :

\begin{align*}
\sum_{n=1}^{T}\Big(\frac{1}{n}\Big)^{\frac{1+\alpha}{2}} \leq 
\begin{cases}
   \frac{2}{1- \alpha} (T)^{\frac{1-\alpha}{2}} + 1 \text{ if  } \alpha <1;\\
  \log T + 1 \text{ if  } \alpha =1 .
\end{cases}
\end{align*} and that 
if $\alpha >1$, 
$$ \sum_{n=  8\log(3 T^{\gamma})}^{T}\Big(\frac{1}{n}\Big)^{\frac{1+\alpha}{2}}  \leq \frac{2 \Big( 8\log(3 T^{\gamma})\Big)^{\frac{\alpha - 1}{2}}}{\alpha-1}.$$ and
$$ \sum_{n=  8\log(3 T^{\gamma})}^{T}\Big(\frac{1}{n}\Big)^{{1+\alpha}}  \leq \frac{\Big( 8\log(3 T^{\gamma})\Big)^{\alpha }}{\alpha}.$$

\subsection{Generalization to a Localized Margin Condition for UCBID and Bernstein-UCBID}\label{sec:localized}

\begin{lemma}\label{cor:localized} If  $F$ satisfies Assumption \ref{ass:margin_cond} with parameter $\alpha$ around $v$ on $[v, v + \Delta]$, the upper bounds presented in Theorems \ref{th:UCBID_strong} and \ref{th:BernsteinUCBID_strong} are unchanged, apart from an additive term 
\begin{align*}
D_T= \begin{cases}
\frac{2 \gamma \log T}{\Delta^2} +1  \text { for the case of UCBID},\\
 \Big(\frac{12}{\Delta} + \frac{32}{\Delta^2}\Big) \log(3T^{\gamma})+1  \text{ for Bernstein-UCBID}.
\end{cases}
\end{align*}
\end{lemma}
\begin{proof}
Alternatively to inequality (\ref{eq:ineq_regret}), we use 
\begin{align*}
R_T & \leq \sum_{t=1}^T \E\left[(M_t-v) \1\{v \leq M_t \leq B_t\} \right]+ \sum_{t=1}^T \Po(B_t<v)  \\
& \leq \sum_{t=1}^T \E \left[(M_t-v) \1\{v \leq M_t \leq B_t< v+ \Delta\}\right]  + \sum_{t=1}^T \Po(B_t<v) 
\\&~~ +  \sum_{t=1}^T \Po\{v+ \Delta \leq B_t, v \leq M_t \leq B_t\}.
\end{align*} 

The first two terms are bounded exactly as in the proofs of Theorems \ref{th:UCBID_strong} and \ref{th:BernsteinUCBID_strong}.
The last term is treated as follows.
\begin{align*}
\sum_{t=1}^T \Po(v+ \Delta \leq B_t, v \leq M_t \leq B_t)
&\leq \sum_{t=1}^T \sum_{n=1}^t \Po (v+ \Delta \leq B^+(n)) \1\{M_t<B_t\}\\
& \leq \sum_{n=1}^t \Po (v+ \Delta \leq B^+(n))\\
&  \leq \sum_{n=1}^t \Po (\Delta -A_n  \leq \bar{V}(n)-v )\\
&  \leq \sum_{n=1}^t \Po (A_n \leq \bar{V}(n)-v ) + \Po(\Delta < 2 A_n),\\
\end{align*}
where  $A_n$ corresponds to $\min\left( 1,2\sqrt{\frac{\gamma\log T}{2n}}\right)$ for UCBID and $ A_n$ is $$ \min \left(1,2 \left(\sqrt{\frac{2 \bar{W}(n) \log(3T^{\gamma})}{n}} +   \frac{3 \log(3 T^\gamma)}{n}\right)\right),$$ for Bernstein-UCBID.\\
The second inequality is proved with similar arguments to those of the proof of Lemma \ref{lem:t_to_n}.
In the case of UCBID, when $n > \frac{2\gamma \log T}{\Delta^2}$, $\Delta \geq 2 A_n$, therefore $\Po(\Delta < 2 A_n) = 0 $ for $n>\frac{2\gamma \log T}{\Delta^2}$. \\
In the case of Bernstein-UCBID, $\Po(\Delta < 2 A_n)\leq \1 \{\Delta < \frac{12 \log(3 T^\gamma)}{n}\} + \Po(\Delta <  \sqrt{\frac{32 \bar{W}(n) \log(3T^{\gamma})}{n}}),$ \\
since $\Delta< a+ b$ implies that $\Delta< 2a$ or $\Delta< 2b$. \\
When $n> 12 \frac{\log(3 T^{\gamma})}{\Delta}$, the first term is equal to 0. The second term is equal to  $\Po(\frac{\Delta^2}{32 \log(3T^{\gamma})} <  \frac{ \bar{W}(n)}{n})$. 
Since $\bar{W}(n)\leq 1$, $\Po(\frac{\Delta^2}{32 \log(3T^{\gamma})}  \frac{ \bar{W}(n)}{n})$  is smaller than $ \1\{n \leq \frac{32}{\Delta^2} \log(3T^\gamma)\}$.\\
This yields 
$$ \sum_{t=1}^T \Po(v+ \Delta \leq B_t, v \leq M_t \leq B_t),\\
 \leq 1 + \begin{cases} \frac{2\gamma \log T}{\Delta^2} \text{ for UCBID}\\  \Big(\frac{12}{\Delta} + \frac{32}{\Delta^2}\Big)  \log 3 T^{\gamma} \text{ for Bernstein-UCBID}.
\end{cases}
$$
\end{proof}

\section{Upper Bound of the Regret of klUCBID}
\label{pr:klucbid_strong}\label{app:klUCBID}
We first prove the following version of Theorem \ref{th:klUCBID} under the assumption that the margin condition holds uniformly on $[v, 1]$. 
We will explain how to adapt the proof to the localized margin condition in Appendix \ref{pr:localized_kl}.
\begin{theorem}\label{th:klUCBID_strong}
  If $F$ satisfies Assumption \ref{ass:margin_cond_uniform} and $F(v)>0$, the kl-UCBID algorithm with parameter
  $\gamma > 1$ yields the following asymptotic bound on the regret:
\begin{align*}
 & \limsup\limits_{T \rightarrow \infty} \frac{R_T}{ (\log T)^{\frac{1+\alpha}{2}} (T)^{\frac{1-\alpha}{2}}} \leq \frac{2\beta}{F(v)(1- \alpha)}\left( 8 \gamma v(1-v))\right)^{\frac{1+ \alpha }{2}} ,\text{ if } \alpha<1. 
&\end{align*}
$$
\limsup\limits_{T \rightarrow \infty} \frac{R_T}{ (\log T)^{2} }
 \leq \frac{\beta}{F(v)} 8 {\gamma v(1-v)}, \text{ if } \alpha \geq 1.
$$
\end{theorem}
\subsection{Analysis under Assumption \ref{ass:margin_cond_uniform}}\label{app:pr_klUCBID_asymp}
We first prove the following lemma.
\begin{lemma}\label{lem:scndtermKL}
 The kl-UCBID strategy ensures that the second term in Lemma \ref{lem:bound_margin_cond} is bounded by a constant:
\begin{align*}
\sum_{t=1}^T \Po( B_t < v) & \leq
  C_{\gamma}. \\
\end{align*}
\end{lemma}

\begin{proof}
We apply Theorem 10 of \cite{garivier2011kl}, and sum over $\{1\ldots T\}$ as in the proof of Lemma \ref{lem:scndterm}.
\end{proof}
Now let us bound the right hand side of inequality \eqref{eq:lem_t_to_n} in Lemma \ref{lem:t_to_n}.We note that 
$$\E\left[ \sum_{n=1}^{T}(B^+(n) - v)_+^{1+ \alpha} \right]
\leq \sum_{n=1}^{T}  \left[\E [ (U_{\gamma}(n)-v )^{1+\alpha}\1\{ L_{\gamma}(n) < v\}] + \Po(v\leq L_{\gamma}(n)) \right]
,$$where $U_{\gamma}(n) = \{q : \bar{V}(n)< q, kl(\bar{V}(n), q)= \frac{\gamma log T}{n}\} = B^+(n)$, and
$L_{\gamma}(n) = \{q : q\leq \bar{V}(n), kl(\bar{V}(n), q)= \frac{\gamma log T}{n}\}$.

Before getting to the proof, we state an important lemma showing that for any $\theta$, there exists a neighborhood in which for all $\hat{\theta},u$, $kl(\hat{\theta}, u)$ is lower bounded by a specific quadratic function of $u- \hat{\theta}$.
\begin{lemma}\label{lem:asymptKL01}
$\forall \theta, ~ \forall \epsilon, \exists \eta_{\epsilon}(\theta)<1, \text{ such that if }\hat{\theta},u  \in [\theta - \eta_{\epsilon}(\theta) , \theta +\eta_{\epsilon}(\theta)],$

$$kl(\hat{\theta}, u)\geq \frac{(u- \hat{\theta})^2}{2\theta(1- \theta)(1 + \epsilon)}.$$  
\end{lemma}
\begin{proof}
Using the Lagrange form of Taylor theorem, 
$$kl(\hat{\theta}, u) \geq \frac{(u- \hat{\theta})^2}{2 \max_{x \in [\hat{\theta},u]} x(1- x)}.$$
By continuity of $x(1-x),$ there exists a neighborhood of $\theta$, in which $x(1-x)\leq (1+ \epsilon) \theta(1- \theta)$. Therefore, if $\hat{\theta}, u$ are in this interval, $\forall x \in  [\hat{\theta},u]$, $x(1-x)\leq (1+ \epsilon) \theta(1- \theta)$.
\end{proof}

 \begin{lemma}
 For $\epsilon>0$ and  $\eta_{\epsilon}$ defined as in Lemma \ref{lem:asymptKL01}, set
$$n_0 (v, T, \epsilon)= \ceil[\Big]{\frac{8 \gamma \log T}{\eta_{\epsilon}^2}}.$$ 
$\text{Then }\forall n\geq n_0(v, T, \epsilon)$,\\
\begin{align*}
\E[(U_{\gamma}(n)-v)_+^{1+\alpha}]
\leq T^{-\gamma} + \Po(v\leq L_{\gamma}(n))  +\left(2\sqrt{\frac{2\gamma \log T v(1-v)(1+\epsilon)}{n}}\right)^{1+\alpha}.
\end{align*}
 \end{lemma}

\begin{proof}

In this proof we will use the notation $n_0$ instead of $n_0(v, T, \epsilon)$  for the sake of clarity.
 We can show that $n_0$ is defined so that $\exp(-\frac{n_0(v, T, \epsilon)}{2}\frac{\eta_{\epsilon}(v)}{4})\leq  \exp(-\frac{n_0(v, T, \epsilon)}{2}\frac{\eta_{\epsilon}^2(v)}{4}) \leq  T^{-\gamma}$.
 Therefore, if $n>n_0$, with probability larger than $1-T^{-\gamma}$, $$|v - \bar{V}(n)|\leq \eta_{\epsilon}/4,$$
 thanks to Hoeffding's inequality.
 Hence, if $n>n_0$, with probability larger than $1-\delta$,
 if $x\in [L_{\gamma}(n), U_{\gamma}(n)]$,
\begin{align*}
|\bar{V}(n)- x | &\leq 
\sqrt{\frac{1}{2}kl(\bar{V}(n), x )}
\leq \sqrt{\frac{\gamma \log T}{2n}}
\leq \sqrt{\frac{\eta_{\epsilon}^2}{4}}
\leq \frac{\eta_{\epsilon}}{2},
\end{align*}
where the first inequality follows from Pinsker's inequality, the second stems from the definition of $L_{\gamma}(n)\text{ and } U_{\gamma}(n)$ and the third one as a result of the definition of $n_0$.
And 
\begin{align*}
|v-x| \leq |\bar{V}(n)- x | + |\bar{V}(n) -v|
\leq \eta_{\epsilon}/2 + \eta_{\epsilon}/4
\leq \eta_{\epsilon}.
\end{align*}  
Finally, using Lemma \ref{lem:asymptKL01},
\begin{align*}
kl(\bar{V}(n), x ) - \frac{(x-\bar{V}(n))^2}{2v(1- v)(1+\epsilon)}\geq 0.
\end{align*}
Hence, with probability larger than 
$1-T^{-\gamma}$, 
\begin{align*}
U_{\gamma}(n)-L_{\gamma}(n) \leq 2 \sqrt{\frac{2\gamma \log T v(1-v)(1+\epsilon)}{n}},
\end{align*}
thanks to  Lemma \ref{lem:asymptKL01}.
Indeed the distance between $U_{\gamma}(n)$ and $L_{\gamma}(n)$ is larger than the distance between the two roots of the quadratic lower bound minus $\frac{\log T}{n}$.
Therefore, 
\begin{align*}
\E[(U_{\gamma}(n)-v)_+^{1+\alpha}]
&\leq
\E[(U_{\gamma}(n)-v)_+^{1+\alpha} \1\{v \in [L_{\gamma}(n), U_{\gamma}(n)]\}] +\Po(v\leq L_{\gamma}(n))\\
&\leq \left(2\sqrt{\frac{2\gamma \log T v(1-v)(1+\epsilon)}{n}}\right)^{1+\alpha}
 +\Po(v\leq L_{\gamma}(n)) + T^{-\gamma}.
\end{align*}

\end{proof}

Since, in fact $B^+(n) = U_{\gamma}(n)$,
\begin{align*} 
& \E\left[ \sum_{n=1}^{T}(B^+(n) - v)_+^{1+ \alpha} \right]\\
&\leq  n_0(v, T, \epsilon) + \sum_{n=n_0(v, T, \epsilon)}^{T} \Po( v \leq L_{\gamma}(n)) + \sum_{n=n_0(v, T, \epsilon)}^{T} T^{-\gamma}  \\&~~+ \sum_{n=n_0(v, T, \epsilon)}^{T}  \left( 2 \sqrt{\frac{2\gamma \log T v(1-v)(1+\epsilon)}{n}} \right)^{1+\alpha}\\
&\leq  n_0(v, T, \epsilon) + 2 \sum_{n=n_0(v, T, \epsilon)}^{T} T^{-\gamma} 
+ \sum_{n=n_0(v, T, \epsilon)}^{T}  \left( 2 \sqrt{\frac{2\gamma \log T v(1-v)(1+\epsilon)}{n}} \right)^{1+\alpha}
\\&\leq  n_0(v, T, \epsilon) + 2  T^{1-\gamma} 
+ \sum_{n=1}^{T}  \left(2  \sqrt{\frac{2\gamma \log T v(1-v)(1+\epsilon)}{n}}  \right)^{1+\alpha},
\end{align*}
where the third inequality comes from the fact that  $\Po(v\leq L_{\gamma}(n)) \leq T^{-\gamma}$ , which is a result of Chernoff's deviation inequality (see Corollary 10.14 of \cite{lattimore2018bandit}). Thus,

\begin{align*}
R_T&\leq C_{\gamma} + n_0(v, T, \epsilon) + 2  T^{1-\gamma} + \sum_{n=1}^{T}  \left(2  \sqrt{\frac{2\gamma \log T v(1-v)(1+\epsilon)}{n}}  \right)^{1+\alpha}\\
\end{align*}
For $\alpha <1$,
We use :
\begin{align*}
\sum_{n=1}^{T}\Big(\frac{1}{n}\Big)^{\frac{1+\alpha}{2}} \leq 
  1+ \frac{2}{1- \alpha} (T)^{\frac{1-\alpha}{2}}  \text{ if  } \alpha <1.\\
\end{align*}

We handle the case $\alpha \geq1$ by observing that when $n \geq 8 \gamma v(1-v)(1+ \epsilon) \log T$,\\ $ \left(2  \sqrt{\frac{2\gamma \log T v(1-v)(1+\epsilon)}{n}}  \right)^{1+\alpha}\leq 8 \frac{\gamma \log T v(1-v)(1+\epsilon)}{n}  $.\\
We also use that 
$$ \sum_{n=  8 \gamma v(1-v)(1+ \epsilon) \log T}^{T}\Big(\frac{1}{n}\Big)^{\frac{1+\alpha}{2}}  \leq \sum_{n=  8 \gamma v(1-v)(1+ \epsilon) \log T}^{T}\frac{1}{n} .$$

Hence,
\begin{align*}
R_T
&\leq C_{\gamma} + n'_0(v, T, \epsilon) + 2   + \begin{cases}
  \frac{\beta}{F(v)}\left( 8 \gamma v(1-v)(1+\epsilon)\log T \right)^{\frac{1+ \alpha }{2}}  \frac{2}{1- \alpha} (T)^{\frac{1-\alpha}{2}} \\
  \text{ if  } \alpha <1\\
   \frac{\beta}{F(v)}8 v(1-v)(1+\epsilon) {\gamma \log T (\log T+1) }\\
     \text{ if } \alpha \geq 1\\
\end{cases}
\end{align*}
where $ n'_0(v, T, \epsilon) = \max(n_0(v,T,\epsilon), 8 \gamma v(1-v) \log T)$.\\

Also, using the fact that $$n_0 = \frac{8 \gamma \log T}{\eta_{\epsilon}^2},$$  $n'_0(v, T, \epsilon) $ is therefore negligible compared to $(\log^2 T)$ and to $T^{\frac{1-\alpha}{2}}$ when $\alpha<1$. And
 $\forall \epsilon > 0,$
\begin{align*}
 \limsup\limits_{T \rightarrow \infty}\frac{R_T}{ (\log T)^{\frac{1+\alpha}{2}} (T)^{\frac{1-\alpha}{2}}} \leq \frac{\beta}{F(v)}\left( 8 \gamma v(1-v)(1+\epsilon)\right)^{\frac{1+ \alpha }{2}}  \frac{2}{1- \alpha}, 
\end{align*}
if $\alpha<1.$
And $\forall \epsilon > 0,$
\begin{align*}
\limsup\limits_{T \rightarrow \infty} \frac{R_T}{ (\log T)^{2} }
 \leq \frac{\beta}{F(v)} 8 {\gamma v(1-v)(1+\epsilon)},
\end{align*}
if $\alpha \geq 1.$\\
Letting $\epsilon$ tend to 0 concludes the proof.

\subsection{Generalization to a Localized Margin Condition for kl-UCBID}\label{pr:localized_kl}
\begin{theorem}
  If $F$ satisfies Assumption \ref{ass:margin_cond} and $F(v)>0$, the kl-UCBID algorithm with parameter
  $\gamma > 1$ yields the following asymptotic bound on the regret:
\begin{align*}
 & \limsup\limits_{T \rightarrow \infty} \frac{R_T}{ (\log T)^{\frac{1+\alpha}{2}} (T)^{\frac{1-\alpha}{2}}} \leq \frac{2\beta}{F(v)(1- \alpha)}\left( 8 \gamma v(1-v))\right)^{\frac{1+ \alpha }{2}} ,\text{ if } \alpha<1. 
&\end{align*}
$$
\limsup\limits_{T \rightarrow \infty} \frac{R_T}{ (\log T)^{2} }
 \leq \frac{\beta}{F(v)} 8 {\gamma v(1-v)}, \text{ if } \alpha \geq 1.
$$
\end{theorem}
In this case,
we adapt the proof in Appendix \ref{app:pr_klUCBID_asymp} by observing that 
\begin{multline*}
R_T \leq \sum_{t=1}^T \E[(M_t-v)\1\{v< M_t< B_t< v+ \Delta \}] \\+ \sum_{t=1}^T \Po (v< B_t) + \sum_{t=1}^T \Po(v+ \Delta \leq B_t, M_t< B_t). 
\end{multline*}
The first two terms of the bound are analyzed exactly in Section \ref{pr:klucbid_strong}.
The last one is bounded as follows
\begin{align*}
\sum_{t=1}^T \Po(v+ \Delta \leq B_t, M_t< B_t)
&\leq  \sum_{n=1}^T \Po(v+ \Delta \leq B^+(n)) \\
& \leq \sum_{n=1}^T \Po(v \leq L_{\gamma}(n)) +  \sum_{n=1}^T \Po (\Delta < B^+(n) - L_{\gamma}(n))\\
& \leq  T^{1-\gamma} + \sum_{n=1}^T \Po (\Delta < U_{\gamma}(n) - L_{\gamma}(n)),
\end{align*}
where the first inequality results from the same steps as in the proof of Lemma \ref{lem:t_to_n}, and the last one stems from the fact that $\Po(v \leq L_{\gamma}(n))\leq T^{-\gamma}$, which is a result of the Chernoff deviation  inequality (see Corollary 10.14 of \cite{lattimore2018bandit}).
For $n> n'_0(v,T,\epsilon)$,
$$U_{\gamma}(n)-L_{\gamma}(n)\leq 2 \sqrt{\frac{2 \gamma v(1-v)(1+ \epsilon) \log T }{n}}$$
and, $$\Delta < U_{\gamma}(n) -L_{\gamma}(n) \implies \Delta < 2 \sqrt{\frac{2 \gamma v(1-v)(1+ \epsilon) \log T}{n}}.$$
Yet this latter condition on $\Delta$ is equivalent to $n< \frac{8 \gamma v(1-v)(1+ \epsilon)}{\Delta^2}\log T \defeq n_1(v, T, \epsilon)$.
This yields, for $T > \max(n'_0(v,T,\epsilon), n_1(v,T,\epsilon))\defeq n_2(v,T,\epsilon),$
\begin{align*}
\sum_{t=1}^T \Po(v+ \Delta \leq B_t, M_t< B_t)
\leq  T^{1-\gamma} + n_2(v,T,\epsilon).
\end{align*}
Hence, 
\begin{align*}
R_T&\leq C_{\gamma}  + n_2(v,T,\epsilon)+ 3 T^{1-\gamma} 
+ \sum_{n=1}^{T}  \left(2  \sqrt{\frac{2\gamma \log T v(1-v)(1+\epsilon)}{n}}  \right)^{1+\alpha}\\
&\leq C_{\gamma}  + n_2(v,T,\epsilon)+ 3  
+ \begin{cases}
  \frac{\beta}{F(v)}\left( 8 \gamma v(1-v)(1+\epsilon)\log T \right)^{\frac{1+ \alpha }{2}}  \frac{2}{1- \alpha} (T)^{\frac{1-\alpha}{2}} \\
  \text{ if  } \alpha <1\\
   \frac{\beta}{F(v)}8v(1-v)(1+\epsilon) {\gamma \log T (\log T+1) }\\
     \text{ if } \alpha \geq 1
\end{cases}.
\end{align*}
Taking the upper limit when $T \rightarrow \infty$ and letting $\epsilon$ tend to 0 concludes the proof.

\subsection{Proof of an Alternative Bound of the Regret of klUCB, under Assumption \ref{ass:margin_cond_uniform}}\label{pr:klucbid_non_asymptotic}

\begin{lemma}\label{lem: klucbid_th_non_asymp}
Under Assumption \ref{ass:margin_cond_uniform},  klUCBID incurs a regret bounded by 
\begin{equation}  R_T \leq 1 + C_{\gamma} + \log T + 2\sqrt{\gamma }\log T +  \begin{cases}
\frac{\beta}{F(v)}
\left(6 + 4 \gamma\right)^2  v \log^2 T, \text{ for } \alpha = 1,\\
 \frac{\beta}{F(v)}
\left(6 + 4 \gamma\right)^{1+\alpha}  \frac{2}{1- \alpha}v^{\frac{1+ \alpha}{2}}\log^{\frac{1+ \alpha}{2}} (T)  T^{\frac{1-\alpha }{2}}\\   +  5 \log^2 T \text{ for } \alpha <1,\\
 \frac{\beta}{F(v)}
\left(6 + 4 \gamma\right)^{1+\alpha}  \frac{2}{\alpha -1 }v^{\frac{1+ \alpha}{2}}\log (T) + 5 \log^2 T \\
 \text{ for } \alpha >1.
 \end{cases} \label{eq:non_asymp_1}.
 \end{equation} 
\end{lemma}
\begin{proof}
\begin{multline}
R_T(v) \leq \sum_{t=1}^T \Po(B_t<v) + \sum_{t=1}^T \E\left[(B_t-v)\1(v<M_t<B_t)\1\left(N_t \leq \frac{\log T}{v} \right)\right]\\
 + \sum_{t=1}^T \E\left[(B_t-v)\1(v<M_t<B_t)\1\left(N_t \geq \frac{\log T}{v} \right)\right] \label{eq:regret_klucbid_decomp}
\end{multline}

The first term of the right hand side of Equation \ref{eq:regret_klucbid_decomp} is bounded by a constant\\ $C_{\gamma}\defeq  \sum_{t=1}^{\infty} e \sqrt{\gamma} \frac{\log t +1}{t^{\gamma}}$ thanks to Theorem 10 of \cite{garivier2011kl}.

We denote by $n_0 = \min\left(T,\floor[\Big]{\frac{\log T}{v}}\right)$.
We start by bounding the second term of the right hand side of Equation \eqref{eq:regret_klucbid_decomp}.

\begin{align*}
&\sum_{t=1}^T (B_t-v)\1(v<M_t<B_t)\1\left(
N_t \leq \frac{\log T}{v} 
\right)\\
&\leq \E\left[ \sum_{t=1}^{T} (B_t-v)\1(v<M_t<B_t)
\1\left(N_t \leq \frac{\log T}{v} \right) \right]\\
&\leq \E\left[
 \sum_{t=1}^{T} 
 \sum_{n=1}^{n_0}
 (B^+(n)-v)
 \1\left(N_t \leq \frac{\log T}{v} \right)
 \1(N_t = n) \1(N_{t+1} = n+1) \right]\\
&\leq  \E
\left[\sum_{n=1}^{n_0}(B^+(n)-v)\right]\\
&\leq  \E\left[ \sum_{n=1}^{n_0}(B^+(n)-v) \1\left(\bar{V}(n)< 2v+ \frac{\log T}{n}\right)   \right] \\
&+~~~\Po\left(\bar{V}(n)\geq 2v+ \frac{\log T}{n}\right)\\
 &\leq  \E\left[ \sum_{n=1}^{n_0}(B^+(n)-\bar{V}(n) + \bar{V}(n)-v) 
\1\left(\bar{V}(n)<  2v+ \frac{\log T}{n}\right)   \right]+ \frac{1}{T},
\end{align*}
where the last line comes from Lemma \ref{lem:benett}.

We know from Lemma \ref{lem:gen_Pinsker_strong} that 
$$\frac{(U_{\gamma}(n) - \bar{V}(n))^2}{2 \times \frac{2((U_{\gamma}(n) - \bar{V}(n) + \bar{V}(n)) + \bar{V}(n)}{3}} \leq kl(\bar{V}(n), U_{\gamma}(n)) = \frac{\gamma \log T}{n}.$$

Hence if we denote by $\Delta(n) := B^+(n) - \bar{V}(n)$,
$$\frac{3 \Delta^2(n)}{2(3 \bar{V}(n) + 2 \Delta(n))}\leq \frac{\gamma \log T}{n}.$$
Therefore, 
\begin{align*}
3 \Delta^2(n)&\leq 2 \gamma \frac{\log T}{n}(3 \bar{V}(n) + 2\Delta(n))\\
&\leq 2 \gamma \frac{\log T}{n}\left(3 (2v + \frac{\log T}{n}) +2 \Delta(n)\right),
\end{align*} when $\bar{V}(n)<  2v+ \frac{\log T}{n}$.
 
Hence 
\begin{align*}
\Delta(n) &\leq \frac{2 \gamma \log T}{3n} + \frac{1}{6} \sqrt{\frac{16 \gamma^2 \log^2 T}{n^2} + \frac{  36 (2v + \frac{\log T}{n})}{n} \frac{2 \gamma \log T}{n}}\\
&\leq \frac{10 \gamma}{3} \frac{\log T}{n} + \sqrt{4 \gamma \frac{v\log T}{n} },
\end{align*}
when $\bar{V}(n)< 2v + \frac{\log T}{n}$.

This yields
\begin{align*}
&\E\left[ \sum_{n=1}^{n_0}(B^+(n)-\bar{V(n)} + \bar{V}(n)-v)\1\left(N_t \geq \frac{\log T}{v} \right) \1\left(\bar{V}(n)< v+ \frac{\log}{n}\right)   \right]\\
&\leq \E\left[ \sum_{n=1}^{n_0} (\Delta(n) + \bar{V}(n)-v)
\1\left(N_t \geq \frac{\log T}{v} \right) \1\left(\bar{V}(n)< v+ \frac{\log T}{n}\right)   \right]\\
&\leq \frac{10 \gamma}{3} \log T \log (n_0) + \sqrt{4 v \gamma} \sqrt{n_0} \sqrt{\log T} + v n_0 + \log T \log(n_0)  \\
&\leq \frac{13 \gamma}{3} \log T \log (n_0)  + \sqrt{4 \frac{v}{v} \gamma \log T} \sqrt{\log T} + \log T.
\end{align*}

Finally, 
\begin{multline}
\sum_{t=1}^T \E \left[(B_t-v)\1(v<M_t<B_t)\1\left(N_t \leq \frac{\log T}{v} \right)\right] \\
\leq \frac{n_0}{T}+ \frac{13 \gamma}{3} \log^2 T + 2\sqrt{ \gamma }\log T + \log T. \label{eq:2ndpart}
\end{multline}

We now bound the third term of the right hand side of Equation \ref{eq:regret_klucbid_decomp}.
\begin{align*}
&\sum_{t=1}^T \E\left[
(B_t-v)\1(v<M_t<B_t)
\1\left(N_t \geq \frac{\log T}{v} \right)\right]\\
& \leq  \sum_{t=1}^T \E \left[
\E \left[
(B_t-v)\1(v<M_t<B_t)
\1\left(N_t \geq \frac{\log T}{v} \right)
 \Big| \mathcal{F}_t \cap \sigma(M_t<B_t)
\right]
\right]\\
& \leq \sum_{t=1}^T \E\left[
\frac{\beta}{F(B_t)}(B_t-v)^{1+\alpha}\1(v<M_t<B_t)\1\left(N_t \geq \frac{\log T}{v} \right)\right]\\
& \leq \sum_{n=n_0 + 1}^{T} \E\left[
\frac{\beta}{F(v)}(B^+(n)-v)^{1+\alpha}\right]\\
&\leq \sum_{n=n_0 +1}^{T} \E\left[
\frac{\beta}{F(v)}\left(B^+(n)-v)\right)^{1+\alpha}
\1\left(\bar{V}(n)-v \leq  \sqrt{2v \frac{\log T}{n}} \right)
\right]  \\
&~~+\Po\left(\bar{V}(n)-v\geq  \sqrt{ 2v \frac{\log T}{n}} \right)\\
& \leq \sum_{n=n_0+1}^{T} \E\left[\frac{\beta}{F(v)} \left(B^+(n)-\bar{V}(n)+ \bar{V}(n) -v)^{1+\alpha}\right)\1\left(\bar{V}(n) -v \leq \sqrt{ 2v \frac{\log T}{n}} \right)\right] + \frac{1}{T},
\end{align*}
where the third inequality comes from a re-sampling argument coupled with the fact that $F(B_t)\geq F(v)$ in case of a won auction, and the last inequality follows from Lemma \ref{lem:benett}.
Like for smaller values of $n$, 
$$\frac{3 \Delta^2(n)}{2(3 \bar{V}(n) + 2 \Delta(n))}\leq \frac{\gamma \log T}{n}.$$
Therefore, 

\begin{align*}
3 \Delta^2(n)&\leq 2 \gamma \frac{\log T}{n}(3 \bar{V}(n) + 2\Delta(n))\\
&\leq 2 \gamma \frac{\log T}{n}\left(3 \left(  v+ \sqrt{2v \frac{\log T}{n}} +2 \Delta(n)\right) \right)\\
&\leq 2 \gamma \frac{\log T}{n}\left(3 \left(  v+ \sqrt{2}v +2 \Delta(n)\right) \right)\\
&\leq 2 \gamma \frac{\log T}{n}\left(3 \left(  (1+ \sqrt{2})v +2 \Delta(n)\right) \right),
\end{align*} 

when $\bar{V}(n) -v <   \sqrt{2v \frac{\log T}{n}}$ and $n\geq n_0$.

This yields 
\begin{align*}
\Delta(n) &\leq \frac{2\gamma\log T}{3n} + \frac{1}{6} \sqrt{\frac{16 \gamma \log^2 T}{n^2} + 36\times  (1+\sqrt{ 2})v\frac{\gamma \log T}{n}}\\
&\leq \frac{4 \gamma \log T}{3n} +  \sqrt{(1+\sqrt{2})v} \sqrt{\frac{\log T}{n}}\\
&\leq \frac{4 \gamma}{3} \sqrt{v} \sqrt{\frac{\log T}{n}} +  \sqrt{(1+\sqrt{2})v} \sqrt{\frac{\log T}{n}}\\
&\leq  \left(2+ \frac{4\gamma}{3}\right) \sqrt{\gamma v} \sqrt{\frac{\log T}{n}},
\end{align*}

when $n>n_0$ and $\bar{V}(n)<   \sqrt{2v \frac{\log T}{n}}$. We used  that $\gamma>1$ in the first inequalities. 
Therefore 
\begin{align*}
&\sum_{n=n_0+1}^{T} \E\left[
\frac{\beta}{F(v)}\left(
B^+(n)-\bar{V}(n)+ \bar{V}(n) -v)
\right)^2
\1\left(
\bar{V}(n)-v\leq \sqrt{2v \frac{\log T}{n}}\right)
\right]\\
&\leq \sum_{n=n_0+1}^{T} \frac{\beta}{F(v)}
\left(
\left(2+ \frac{4 \gamma}{3}\right)
\sqrt{ \frac{ \gamma v \log T}{n}} + \sqrt{2 v \frac{\log T}{n}}
 \right)^{1+\alpha}\\
&\leq \sum_{n=n_0+1}^{T} \frac{\beta}{F(v)}
\left(6 + 4 \gamma\right)^{1+ \alpha}  \left(v \frac{\log T}{n}
 \right)^{\frac{1+\alpha}{2}}\\
 &\leq  \frac{\beta}{F(v)}
\left(6 + 4 \gamma\right)^{1+\alpha} \begin{cases}v \log T(\log T - \log(n_0)) \text{ if } \alpha = 1\\
\frac{2}{\alpha -1}v^{\frac{1+ \alpha}{2}}\log^{\frac{1+ \alpha}{2}} (T)  T^{\frac{1-\alpha}{2}} \text{ if } \alpha < 1 \\ 
\frac{2}{1- \alpha }v^{\frac{1+ \alpha}{2}}\log^{\frac{1+ \alpha}{2}} (T) \log^{\frac{ 1- \alpha}{2}} (n_0),  \text{ if } \alpha > 1\\
\end{cases}.
\end{align*}

When using Equation \eqref{eq:regret_klucbid_decomp} and hence summing with $C_{\gamma}$ and the right hand side of Equation \eqref{eq:2ndpart}, we obtain:
$$R_T \leq 1 + C_{\gamma}   + 2\sqrt{ \gamma}\log T + \log T+  \frac{\beta}{F(v)}
\left(6 + 4 \gamma\right)^2  v \log^2 T,$$
for $\alpha = 1$, using the fact that $\frac{10 \gamma}{3}< \left(6 + 4 \gamma\right)^2 $, to make the terms proportional to $\log(n_0)$ disappear. Similarly,
$$R_T \leq 1 + C_{\gamma} + \frac{13 \gamma}{3} \log^2 T +  2\sqrt{ \gamma}\log T + \log T +  \frac{\beta}{F(v)}
\left(6 + 4 \gamma\right)^{1+\alpha}  \frac{2}{1- \alpha }v^{\frac{1+ \alpha}{2}}\log^{\frac{1+ \alpha}{2}} (T)  T^{\frac{1-\alpha }{2}} ,$$
for $\alpha < 1$.
And 
$$R_T \leq 1 + C_{\gamma} + \frac{13 \gamma}{3} \log^2 T   + 2\sqrt{ \gamma}\log T + \log T+ \frac{\beta}{F(v)}
\left(6 + 4 \gamma\right)^{1+\alpha} \frac{2}{\alpha -1 }v^{\frac{1+ \alpha}{2}} \log(T),$$ for $\alpha > 1$.

\end{proof}

\section{Proof of the Lower Bound for Optimistic Strategies}\label{pr:lower_bound}

We restate the Theorem for simplicity.
\begin{repeatthm}{th:lower_bound} 
  We consider all environments where $V_t$ follows a Bernoulli distribution with expectation $v$ and $F$ admits a density $f$ that is both upper bounded and lower bounded, with $f(b)\geq \ubar{\beta}>0$.
If a strategy is such that, for all such environments,
$R_T\leq O(T^{a})$, for all $a>0$, 
and there exists $\gamma >0$ such that for all such environments, $\Po(B_t< v)<t^{-\gamma}$,
 then this strategy must satisfy:
$$\liminf_{T \rightarrow \infty} \frac{R_T}{\log T} \geq \ubar{\beta} \frac{v}{16 F(v)}\;.$$
\end{repeatthm}
\begin{proof}
We need the following lemma:
\begin{lemma}\label{lem:limit_Nt}
If $~ R_T \leq O(T^{a}), ~ \forall a>0,$ and $F$ admits a density which is lower bounded by a positive constant and upper bounded. 
Then, $$ \lim_{t \rightarrow \infty} \E \left[\frac{N_t}{t}\right] = F(v).$$ 
\end{lemma}
\begin{proof}
$\sum_{t=1}^{T} \E[(B_t- v)^2]\leq O(T^{a}), ~ \forall a>0,$ because of Lemma \ref{lem:density_f}.
By the tower rule, it holds $\E\left[\frac{N_t}{t}\right]= \E[\frac{1}{t}\sum_{s=1}^t F(B_s)]$.
Since $F$ admits a density $f$, upper bounded by a constant that we denote $\beta$, 
\\
$$\E[(F(B_t) - F(v))^2]]\leq \beta^2 \E[(B_t-v)^2] .$$
Since  $R_T \leq O(T^{a}), ~ \forall a>0$, then $\sum_{t=1}^{T} \E[(B_t- v)^2]\leq O(T^{a}), ~ \forall a>0$, in particular  $\lim_{t \rightarrow \infty}\E[(B_t- v)^2] = 0 $.
Combining these two arguments yields $\lim_{t \rightarrow \infty}\E[(F(B_t) - F(v))^2]=0$ and since $L_2$-convergence implies $L_1$-convergence, $\lim_{t \rightarrow \infty}\E[F(B_t)] = F(v)$. \\
Together with  the fact that $\E\left[\frac{N_t}{t}\right]= \E[\frac{1}{t}\sum_{s=1}^t F(B_s)]$, and with the Cesaro theorem this result proves the lemma.
\end{proof} 
We get back to the proof of Theorem \ref{th:lower_bound}.
We set a time step $t\in [1,T]$.	We consider two alternative environments with identical distributions for $M_t$ that differ by the distribution of $V_t$. The value $V_t$ is distributed according to a Bernoulli distribution of expectation $v$ in the first environment, respectively $v'_t = v + \sqrt{\frac{v(1-v)}{F(v)t}}$, in the second environment.

\paragraph{Notation.}
We will denote by $\Po_v(\cdot)$ the probability of an event under the first environment (respectively $\E_v(\cdot)$ the expectation of a random variable under the first environment), and by $\Po_{v'_t}(\cdot)$ the probability of an event under the second environment (respectively $\E_{v'_t}(\cdot)$ the expectation of a random variable under the first environment).
We denote by $I_t$ the information collected up to time $t+1$ : $(M_t, V'_t, \ldots M_1, V'_1)$.
Finally, $\Po_v^{I_t}$ (respectively $\Po_{v'_t}^{I_t}$) is the law of $I_t$ in the first (respectively second) environment.

We consider the Kullback Leibler divergence between $\Po_v^{I_t}$ and $\Po_{v'_t}^{I_t}$,
$$ KL(\Po_v^{I_t},\Po_{v'_t}^{I_t})= kl(v,v'_t) \E[N_t]$$
This is proved using the chain rule for conditional KL. 

In fact, 
$$
KL(\Po_v^{I_t},\Po_{v'_t}^{I_t})= KL(\Po_v^{I_{t-1}},\Po_{v'_t}^{I_{t-1}})\\ + KL(\Po_v^{(M_t,V'_t)|I_{t-1}},\Po_{v'_t}^{(M_t,V'_t)|I_{t-1}}),
$$
and
\begin{align*}
KL(\Po_v^{(M_t,V'_t)|I_{t-1}},\Po_{v'_t}^{(M_t,V'_t)|I_{t-1}})&= \E[\E[KL(\nu_{I_t}\otimes \mathcal{D}, \nu'_{I_t}\otimes \mathcal{D})|I_{t-1}]]\\
&= \E[kl(v, v')\1(B_t>M_t)].\\
\end{align*}
where $\nu_{I_t}$(respectively $\nu'_{I_t}$) denotes the  law of $V'_t$ knowing $I_t$ in the first environment (respectively the second), and $\mathcal{D}$ the law of $M_t$.

By induction, we obtain $$ KL(\Po_v^{I_t},\Po_{v'_t}^{I_t})= kl(v,v'_t) \E[N_t].$$

Using Lemma \ref{lem:limit_Nt},
$\forall \epsilon>0, \exists t_1(\epsilon), \forall t\geq t_1(\epsilon) $, 
$$ KL(\Po_v^{I_t},\Po_{v'_t}^{I_t}) \leq kl(v,v'_t)(1+\epsilon)F(v).$$
Using the data processing inequality (see for example \citet{garivier2019explore}), we get
\begin{align*}
 KL(\Po_v^{I_t},\Po_{v'_t}^{I_t})
& \geq kl \left(\Po_v \left( B_t > \frac{v+ v'_t}{2}\right), \Po_{v'_t} \left( B_t > \frac{v+ v'_t}{2}\right)\right)\\
& \geq  2 \left(\Po_v \left( B_t > \frac{v+ v'_t}{2}\right)- \Po_{v'_t} \left( B_t > \frac{v+ v'_t}{2}\right)\right)^2\\
& \geq  2 \left(\Po_v \left( B_t > \frac{v+ v'_t}{2}\right)+ \Po_{v'_t} \left( B_t < \frac{v+ v'_t}{2}\right)-1\right)^2,
\end{align*}
where the  second inequality comes from Pinsker inequality.
Therefore, 
\begin{align*}
\Po_v \left( B_t > \frac{v+ v'_t}{2}\right)+ \Po_{v'_t} \left( B_t < \frac{v+ v'_t}{2}\right)
 \geq  1 - \sqrt{\frac{1}{2} KL(\Po_v^{I_t},\Po_{v'_t}^{I_t})}.
\end{align*}
Specifically, $\forall t>t_0(\epsilon)$,
\begin{align*}
\Po_v \left( B_t > \frac{v+ v'_t}{2}\right)+ \Po_{v'_t} \left( B_t < \frac{v+ v'_t}{2}\right)
\geq  1 - \sqrt{\frac{1}{2}kl(v,v'_t)(1+\epsilon)F(v)t}.
\end{align*}

Using  the fact that $\E_v[(B_t-v)^2]\geq  (v-\frac{v+ v'_t}{2})^2\Po_v \left( B_t > \frac{v+ v'_t}{2}\right)$ yields
\begin{align*}
\E_v[(B_t-v)^2]&\geq  \left(\frac{v- v'_t}{2}\right)^2 \Po_v \left( B_t > \frac{v+ v'_t}{2}\right)\\
& \geq \frac{v(1-v)}{4F(v)t} \left( 1- \sqrt{ \frac{1}{2}(1 + \epsilon)  kl(v, v'_t)F(v)t}   - 1/ {t^{\gamma}}\right),
\end{align*}
using the assumption that the algorithm outputs a bid that does not underestimate $v'_t$:  $\Po_{v'_t}(B_t< v'_t)<\frac{1}{t^{\gamma}}$.

We use the fact that
$\forall \epsilon>0, ~\exists t_2(v, \epsilon), ~\forall t\geq t_2(v,\epsilon),~  kl\left(v, v + \sqrt{\frac{v(1-v)}{F(v)t}}\right) \leq \frac{1 + \epsilon}{2F(v)t}  $
which is proved with similar arguments to those used to prove Lemma \ref{lem:asymptKL01}.

Altogether, we have proved\\
$\forall t\geq \max( t_1(\epsilon),t_2(v,\epsilon)), $
\begin{align*}
\E_v[(B_t-v)^2] \geq \frac{v(1-v)}{4F(v)t} \left( 1- \sqrt{ \frac{1}{4}(1 + \epsilon)^2  }   - 1/ t^{\gamma}\right).
\end{align*}
Let $t_0(v,\epsilon) =\max( t_1(\epsilon),t_2(v,\epsilon)).$
We obtain
\begin{align*}
\sum_{t= 1}^T \E_v[(B_t-v)^2] \geq \sum_{t= t_0(v,\epsilon)}^T \frac{v(1-v)}{4F(v)t} \left( 1-  \frac{1}{2}(1 + \epsilon)   - 1/ t^{\gamma}\right).
\end{align*}

Recall that, according to Lemma \ref{lem:density_f}, 
$$R_T(v) \geq \frac{\ubar{\beta}}{2} \sum_{t= 1}^T \E_v[(B_t-v)^2].$$

Hence, $\forall \epsilon>0,$

$$
 R_T(v)
 \geq \frac{\ubar{\beta}}{2} \left(\frac{v(1-v)}{4} \left( 1-  \frac{1}{2}(1 + \epsilon)\right)\right) \log \frac{T}{t_0(v,\epsilon)}  - O(1).
$$

And $\forall \epsilon>0,$
\begin{align*}
& \liminf_{T \rightarrow \infty}  \frac{R_T(v)}{\log T}\geq \frac{\ubar{\beta}}{2} \left(\frac{v(1-v)}{4F(v)} \left( 1-  \frac{1}{2}(1 + \epsilon)\right)\right) .
\end{align*}
Since this holds for all $\epsilon$,
\begin{align*}
& \liminf_{T \rightarrow \infty}  \frac{R_T(v)}{\log T}\geq \ubar{\beta} \left(\frac{v(1-v)}{16F(v)} \right).
\end{align*}
\end{proof}
\section{Proof of Theorem \ref{th:other_strats}}\label{pr:other_strats}
The statement of Theorem  \ref{th:other_strats}is repeated here for simplicity.
\begin{repeatthm}{th:other_strats}
Without further assumption, the maximal regrets of UCBID, BernsteinUCBID and klUCBID are $O(\sqrt{T}\log T)$. If $F$ has a density that is bounded from below and above by non negative constants, the maximal regret of UCBID remains of the same order, while it is reduced to $O(T^{\frac{1}{3}}\log^2 T)$ for BernsteinUCBID and to $O(\log^2 T)$ for klUCBID.
\end{repeatthm} 
 \begin{proof}
We first start by bounding the regret of UCBID.
The UCBID strategies incurs a regret bounded by:
\begin{align*}
E_v[R_T]\leq& \E_v \left[ \sum_{t=1}^T (v-B_t) \1(v<M_t<B_t)\right] +\sum_{t=1}^T \Po(v>B_t)\\
&\leq \E_v\left[\sum_{t=1}^T 2\sqrt{\frac{\gamma \log t}{2 N_t}}\1(M_t<B_t)\right] +\sum_{t=1}^T \Po\left(\bar{V}_t>v + \sqrt{\frac{\gamma \log t}{2 N_t}} \right) +\sum_{t=1}^T \Po(v>B_t) \\
&\leq \E_v\left[\sum_{n=1}^T 2\sqrt{\frac{ \gamma \log T}{2 n}}\right] +\sum_{t=1}^T \Po \left(\bar{V}_t>v + \sqrt{\frac{\gamma \log t}{2 N_t}} \right) +\sum_{t=1}^T \Po(v>B_t) \\
&\leq 2\sqrt{\gamma T \log T  }+2C_{\gamma},
\end{align*}
where the third inequality follows from a resampling argument close to that of Lemma  \ref{lem:t_to_n}.
By Pinsker's inequality, this bound also trivially holds for klUCBID.
Along with the bound of Theorem \ref{th:UCBID}, this suggests that the point where the maximal regret of UCBID is reached is $O(T^{-\frac{1}{2}})$, under Assumption \ref{ass:bounded_dens}.

Next, we prove the bound for the regret of BernsteinUCBID.
We proved in Section \ref{pr:bernsteinucbid_strong}, that the regret of BernsteinUCBID satisfies :
$$R_T\leq C'_{\gamma} +\frac{\beta}{F(v)} 8  w \log (3 T ^{\gamma}) (\log T+1)\\
   +\frac{\beta}{F(v)} \Big((6 c_1 \log (3 T ^{\gamma}))^{2} +1\Big).$$
If $v> T^{-\frac{1}{3}}$, this entails :
$$R_T\leq C'_{\gamma} +8 \frac{\beta}{\ubar{\beta}}  \log (3 T ^{\gamma}) (\log T+1)
   +\beta T^{\frac{1}{3}}\Big((6 c_1 \log (3 T ^{\gamma}))^{2} +1\Big).$$
If, on the other hand, $v\leq  T^{-\frac{1}{3}}$,
\begin{align}
E_v[R_T] &\leq  \E_v\left[ \sum_{t=1}^T (v-B_t) \1(v<M_t<B_t)\right] +\sum_{t=1}^T \Po(v>B_t)
\\& \leq \sum_{t=1}^T \E_v\left[2 \sqrt{\frac{2 \bar{W}_{t-1} \log(3t^{\gamma})}{N_{t-1}}} + \frac{3 \log(3 t^\gamma)}{N_{t-1}}\right]  \notag
\\&~~~~ +\sum_{t=1}^T  \Po\left(\bar{V}_t- v \leq \sqrt{\frac{ \bar{W}_{t-1} \log(3t^{\gamma})}{N_{t-1}}} + \frac{3 \log(3 t^\gamma)}{N_{t-1}}\right) +\sum_{t=1}^T \Po(v>B_t) \notag
\\&  \leq \sum_{n=1}^T \E_v\left[2 \sqrt{\frac{2 \bar{W}(n) \log(3t^{\gamma})}{n}} + \frac{3 \log(3 T^\gamma)}{n}\right] + 2C'_{\gamma} \notag
\\& \leq  \sum_{n=1}^T 2 \sqrt{\frac{2 w \log(3T^{\gamma})}{n}} + \frac{3 \log(3 T^\gamma)}{n} + 2C'_{\gamma} \notag
\\&\leq 2 \sqrt{2 w \log(3T^{\gamma})T} + 3 \log( 3T^{\gamma}) \log T + 2C'_{\gamma} \label{ineq:general_bernstein_wc}
\\&\leq 2 \sqrt{2  \log(3T^{\gamma})} T^{\frac{1}{3}} + 3 \log( 3T^{\gamma}) \log T + 2C'_{\gamma} \notag
\end{align}

where the third inequality comes from a resampling argument close to the one proved in the proof of Lemma \ref{lem:t_to_n}, the fourth from the Jensen inequality, and the fifth from $w \leq v(1-v)$.

In either case, $R_T(v)\leq O(T^{\frac{1}{3}}\log^2 T)$. This suggests that the point where the maximal regret of BernsteinUCBID is reached is a $O(T^{-\frac{1}{3}})$.

Note that without any assumption on $F$ and for any $v\in [0,1]$ ,
Inequation \ref{ineq:general_bernstein_wc} still holds. Since $w<1$ and 
$\sqrt{\log(3T^{\gamma})}= O (\log{T})$, we prove that 
$\E_v[R_T]\leq O(\sqrt{T} \log{T}$, without any assumpton on $F$.

The bound of the worst-case regret of klUCBD under Assumption \ref{ass:bounded_dens} directly stems from Lemma \ref{lem: klucbid_th_non_asymp}.

\end{proof}

\section{Proof of Theorem \ref{th:ETGstop}} \label{pr:ETGstop}
In the main body of the paper, we stated the following theorem.
\begin{repeatthm}{th:ETGstop}
 If $F$ admits a density $f$, that satisfies 
$\exists ~ \ubar{\beta}, \beta>0, \forall x\in [0,1],~ ~ \ubar{\beta} \leq f(x) \leq \beta;$ Then the regret of ETGstop satisfies :
$$\max_{v \in [0,1]} R_T(v)\leq O(T^{\frac{1}{3}}\log^2 T),$$
$$ \text{and if }v> \frac{1}{T^{\frac{1}{3}},}\text{ then}~~~~~ R_T\leq 7 + \frac{64\log(T) + 60T^{-1/2}}{v}+ \frac{4}{F(v/2)} + \beta \frac{ \log^2 T}{F(v/2)}.$$ 
\end{repeatthm}

We recall that ETGstop is defined by the following choice of stopping times $\tau_1$ and $\tau_0$: \\
\begin{equation}\label{def:tau_1bis}\tau_1 := \inf\left\{t \in [1,T]: \exp(-\frac{t L_t}{8})\leq \frac{1}{T^2}\right\}, ~ \tau_0 =  \inf\left\{t \in [1,T]: U_t \leq \frac{1}{T^{\frac{1}{3}}}\right\}
\end{equation}
where
$L_t= \min\{ v\in [0, \bar{V}_t[: \exp(- t kl(\bar{V}_t,v)) \leq
{1}/{T^2}\}$ and\\
$U_t= \max\{ v\in [\bar{V}_t,1[: \exp(- t kl(\bar{V}_t,v)) \geq
{1}/{T^2}\}$.

We use the fact that this is completely equivalent to choosing the following stopping times
\begin{equation}\label{def:rho_1}\rho_1 = \inf\left\{n \in [1,T]: \exp\left(-\frac{n L(n)}{8}\right)\leq \frac{1}{T^2}\right\} , ~ \rho_0 =  \inf\left\{n \in [1,T]: U_{\gamma}(n) \leq \frac{1}{T^{\frac{1}{3}}}\right\},
\end{equation}
where the stopping times are defined from the number of observations rather than the number of auctions, because, in the first phase the bidder always observes the value of the item, as $B_t = 1$.
Therefore, in particular, $\min(\rho_1, \rho_2)= \min(\tau_0, \tau_1)$.
Note however that in general $\max(\rho_1, \rho_2)\neq \max(\tau_0, \tau_1)$. In particular, when $\tau_0$ is reached first, $\bar{V}_t= \bar{V}_{\tau_0}, \forall t>\tau_0$ because there is no further observation once $\tau_0$ is reached, while this is not the case of $\bar{V}(n)$ : this means that $L_t$ will be constant after  $\tau_0$ while $L(n)$ is not necessarily constant after $\tau_0$. 

In the sequel, we will use these stopping times instead of $\tau_0$ and $\tau_1$.

\subsection{Preliminary lemmas}
The stopping time $\tau_1$ (respectively $\rho_1$) is designed so that if it occurs before $\tau_0$ (respectively $\rho_0$) then  the bids in the second phase will be larger than $v/2$ with high probability. We show that if all bids in the second phase are larger than  $v/2$, the regret in the second phase is bounded as follows.
\begin{lemma}\label{lem:reg_fav_event}
If $F$ admits a density bounded by $\beta$
then
\begin{align*} 
 \E \left[ \sum_{t=\rho_1}^T r_t \1\left(\cap_{s=\rho_1}^T \{\bar{V}(s)\geq \frac{v}{2}\}\right)\1(\rho_1< \rho_0) \right]
 & \leq  \frac{ 4 +\beta \log^2 T}{F(v/2)}+ 1. 
\end{align*}
\end{lemma}
\begin{proof}
We denote by  $\mathcal{A}$ the advantageous event $\cap_{s=\rho_1}^T \{\bar{V}(s)\geq \frac{v}{2}\}$. 
We first observe that:
$$\rho_1< \rho_0  \text{ and } \mathcal{A}  \implies \forall t \in [1, T] ~ , B_t \geq  \frac{v}{2}.$$
Indeed,  in the first phase, $B_t = 1$, and in the second phase, $B_t = \bar{V}_t = \bar{V}(N_t)$, with $N_t \geq \rho_1$.
Hence it holds with high probability that if $t\geq \rho_1,  ~ \forall c>1,~~
~ N_t\geq \rho_1 + \frac{F(v/2)}{c} (t- \rho_1)$.\\
Indeed, a strategy playing $v/2$ instead of $\bar{V}_t$ in this phase would obtain  $N'_t$ victories, where $N'_t$ is the sum of $\rho_1$, and of $t-\rho_1$ Bernoullis of expectation $F(v/2)$, and it holds that $N_t\geq N'_t$, as a larger bid implies a larger number of won auctions. 
Therefore, if $t \geq \rho_1$, the probability that $N_t< \rho_1 + \frac{F(v/2)}{c} (t- \rho_1)$ conditioned on 
$\mathcal{A} \cap \left\{ \rho_1<\rho_0 \right\}$ can be bounded as follows.
\begin{align*} &\Po \left(N_t<\rho_1 + \frac{F(v/2)}{c} (t- \rho_1)\Big| \mathcal{A} \cap \left\{ \rho_1< \rho_0 \right\} \right)\\ & \leq \Po \left(N'_t<\rho_1 + \frac{F(v/2)}{c} (t- \rho_1)\right)
\\ &=\E \left[ \Po \left( F(v/2) (t- \rho_1) -(N'_t - \rho_1) > \frac{F(v/2)(c-1)}{c} (t- \rho_1)\Big|\rho_1\right)\right] \\
& \leq  \E\left[\exp \left(-\frac{2(c-1)^2(F(v/2))^2}{c^2}(t- \rho_1)\right)\right]\\
& \leq \exp \left(-\frac{2(c-1)^2(F(v/2))^2}{c^2}t\right),
\end{align*}
where we used Hoeffding's inequality for the second inequality.
Now, we can use the following :
\begin{align} &\E\left[r_t \1\{t\geq \rho_1 \}\1\left(\mathcal{A}\right) \1( \rho_1< \rho_0)\right]\\
& \leq \E\left[r_t \1\{t\geq \rho_1 \} \1\left\{N_{t-1}\geq \rho_1 + \frac{F(v/2)}{c} (t- \rho_1-1)\right\}\1( \rho_1< \rho_0)\right] \notag \\
&~~~~+\Po \left(N_{t-1}<\rho_1 + \frac{F(v/2)}{c} (t- \rho_1-1), t \geq \rho_1 \1( \rho_1< \rho_0 )\right) \notag \\
& \leq \E\left[
r_t \1\{t\geq \rho_1 \}  \1
\left\{N_{t-1}\geq \rho_1 + \frac{F(v/2)}{c} (t- \rho_1-1)
\right\} \1(\rho_1< \rho_0 )
\right] \label{ineq_term0} \\
& +~~\exp \left(-\frac{2(c-1)^2(F(v/2))^2}{c^2}t\right) 
\end{align}
The first term of the right hand side of \eqref{ineq_term0}  is bounded by :
\begin{align*}
&\E\left[r_t \1\{t\geq \rho_1 \}  \1\{N_{t-1}\geq \rho_1 + \frac{F(v/2)}{c} (t- \rho_1-1)\} \1( \rho_1< \rho_0)\right]  \\
& \leq \E \left[\E\left[\E\left[r_t 
\1\left\{t\geq \rho_1 \right\} 
 \1\left\{ N_{t-1}\geq \rho_1 + \frac{F(v/2)}{c} (t- \rho_1-1)
 \right\}
\Big|\mathcal{F}_{t-1}
\right]
\right]\1( \rho_1< \rho_0)
\right] \\
& \leq \E\left[
\frac{\beta}{2}(v-B_t)^2 \1 \{ t\geq \rho_1 \} \1\left\{
N_{t-1}\geq \rho_1 + \frac{F(v/2)}{c} (t- \rho_1-1)
\right\}1( \rho_1< \rho_0)
\right] ,
\end{align*}
where the last line follows from Lemma \ref{lem:bound_margin_cond}.
This is also clearly bounded by
\begin{align}
& \E\left[\frac{\beta}{2}(v-B_t)^2  \1\left\{t\geq \rho_1 \right\} \1\left\{N_{t-1}\geq \rho_1 + \frac{F(v/2)}{c} (t- \rho_1-1)\right\} \1( \rho_1< \rho_0)\right] \notag\\
& \leq \E \left[\frac{\beta}{2} \frac{ \log T}{ N_{t-1}} \mathbb{1}\{N_{t-1}\geq \rho_1 + \frac{F(v/2)}{c} (t- \rho_1-1 )\} \1\{t\geq \rho_1 \} \1( \rho_1<\rho_0)\right] \notag\\
&~~~+ \Po \left(\left|v- \bar{V}(N_{t-1})\right|\geq \sqrt{\frac{ \log T}{N_{t-1}}}  ,t\geq \rho_1 \right) \label{ineq : hoeffing_bound}.
\end{align}
By applying a union bound, followed by Hoeffding's inequality,the second term of \eqref{ineq : hoeffing_bound} is bounded as follows
$$\Po \left(|v- \bar{V}(N_{t-1})|\geq \sqrt{\frac{ \log T}{N_{t-1}}}  ,t\geq \rho_1 \right) \\ 
\leq \sum_{s=1}^{t-1} \Po \Big(|v- \bar{V}(s)|\geq \sqrt{\frac{ \log T}{ s}}  ,t\geq \rho_1 \Big)\\
\leq \frac{1}{T}.
$$
Summarizing,
\begin{multline*}\E\left[r_t
 \1\left\{t\geq \rho_1 ,
 	\left\{N_{t-1}\geq \rho_1 + \frac{F(v/2)}{c} (t- \rho_1-1)
 	\right\},
 	\rho_1< \rho_0)
 	\right\}\right] \\
\leq \E\left[\frac{\beta}{2} \frac{c \log T}{F(v/2)t} \1\{t\geq \rho_1 \}\right] + \frac{1}{T}.
\end{multline*}
Hence, when summing over $T$ rounds, 
\begin{align*}
 \E \left[ \sum_{t=\rho_1}^T r_t  \1\left(\cap_{s=\rho_1}^T \{\bar{V}(s)\geq \frac{v}{2}\}\right) \1( \rho_1< \rho_0 )\right]
 & \leq \sum_{t=1}^T\E\left[\frac{\beta}{2} \frac{c \log T}{F(v/2)t} \1\{t\geq \rho_1 \}\right] + \frac{1}{T} \\
 &~~~+ \exp \left(-\frac{2(c-1)^2(F(v/2))^2}{c^2}(t- \rho_1)\right)\\
 & \leq  \frac{\beta}{2} \frac{c \log^2 T}{F(v/2)t}+  1 + \frac{1}{1- \exp(-2\frac{2(c-1)^2(F(v/2))^2}{c^2})}\\
 &\leq  \frac{\beta}{2} \frac{c \log^2 T}{F(v/2)t}+  1 + \frac{c^2}{(c-1)^2F(\frac{v}{2})}.
\end{align*}
Picking $c= 2$ concludes the proof.
\end{proof}

We claim that the stopping time $\tau_1$ (respectively $\rho_1$) is designed so that if it occurs before $\tau_0$ (respectively $\rho_0$) then  the bids in the second phase will be larger than $v/2$ with high probability. The following lemma quantifies this statement.

 \begin{lemma}\label{lem:prob_unfav_event} The event $\mathcal{A} := \cap_{s=\rho_1}^T \{\bar{V}(s)\geq \frac{v}{2}\}$ occurs with high probability and
$$
 \Po \left( \mathcal{A}^c\right) 
 \leq \frac{2}{T}.$$
\end{lemma}

\begin{proof}
By a union bound, 
\begin{align}
 \Po \left( 
 \mathcal{A}^c\right) &\leq \Po\left(v \leq L_{\rho_1}\right) + \sum_{n=1}^T \Po \left( \left\{v \geq L_{\rho_1}\right\}\cap \left\{\bar{V}(n)< \frac{v}{2}\right\}\cap \left\{ n\geq \rho_1 \right\} 
 \right)\label{ineq:lem_unfav_event}
\end{align}
The second term of Equation \ref{ineq:lem_unfav_event} is bounded by:
\begin{align}
\Po\left(
 \left\{v \geq L({\rho_1})\right\}
 \cap\left\{\bar{V}(n)< \frac{v}{2}\right\}
 \cap \left\{ n\geq \rho_1\right\} 
 \right) &
 \leq \Po \left( 
 \left\{\exp(-\frac{nv}{8}) \leq \frac{1}{T^2}\right\}
 \cap
 \left\{\bar{V}(n)<\frac{v}{2}\right\}
 \right) \notag \\
&\leq \1 \left(
\exp(-\frac{nv}{8}) \leq \frac{1}{T^2}
\right) \Po \left(
\bar{V}(n)<\frac{v}{2}
\right) \notag \\
&\leq \1 \left(\exp(-\frac{nv}{8}) \leq \frac{1}{T^2}\right) \exp(-\frac{nv}{8}) \notag\\
&\leq \frac{1}{T^2}\label{ineq:unfav_event_1}.
\end{align}
where we use Lemma \ref{lem:kl} along with Chernoff inequality to prove that $\forall n, \Po(\bar{V}(n)<\frac{v}{2})\leq \exp(-3\frac{ nv}{20})\leq \exp(-\frac{ nv}{8})$.

Let us go back to the probability of the unwanted event $\mathcal{A}^c$.
\begin{align*}
 \Po \left( \cup_{s=\rho_1}^T \left\{\bar{V}(s)< \frac{v}{2}\right\}\right) 
 &\leq \Po(v \leq L(\rho_1)) + \sum_{t=1}^T \Po \left( \{v \geq L(\rho_1)\}\cap\left\{\bar{V}(s)< \frac{v}{2}\right\}\cap \{ t\geq \rho_1\} \right)\\
 &\leq \frac{1}{T} + \frac{1}{T}\\
 &\leq \frac{2}{T},
\end{align*}
by another union bound and thanks to Equation \eqref{ineq:unfav_event_1}. 
\end{proof}

The expectation of $\rho_1$ and $\rho_0$ are of critical importance. Indeed, the regret is larger than the expectation of their minimum multiplied by $\ubar{\beta}(1-v)^2/2$. In the following lemma, we show in particular that the expectation of $\tau_1$ is proportional to the inverse of $v$.
\begin{lemma}\label{lem:exp_tau} The expectation of $\rho_1$ is bounded by
 $$\E[\rho_1]\leq \frac{64\log(T) + 60T^{-1/2}}{v}+ 1.$$
\end{lemma} 
\begin{proof}
First note that by Pinsker's inequality
\begin{equation} L(n) \geq \bar{V}(n) -\sqrt{\frac{\log(T)}{2n}}
\end{equation} 
and  by the generalized Pinsker inequality,
\[\mathrm{kl}(p,q) \geq \frac{(p-q)^2}{2\max\{r(1-r) : r\in(p,q)\}}\;.\]
 If $\bar{V}(t)\leq 1/2$ then 
\begin{equation}
 L(n) \geq \bar{V}(n) -\sqrt{\frac{2\bar{V}(n)\log(T)}{n}}\;.
\label{eq:csq:pinsker:gal}
\end{equation}

Let $n_0 = \lceil 64\log(T)/v \rceil$, and $n\geq n_0$.
\begin{itemize}
	\item If $\bar{V}(n)\geq 1/2$, then 
	\[L(n) \geq \bar{V}(n) -\sqrt{\frac{\log(T)}{2n}}  \geq \frac{1}{2} -\sqrt{\frac{\log(T)}{16\log(T)}} =\frac{1}{4} ,\]
and hence $nL(n) \geq n/4 \geq 16\log(T)$.

\item If $v\leq \bar{V}(n)\leq 1/2$, then 
\[L(t) \geq \bar{V}(n) - \sqrt{\frac{2\bar{V}(n)\log(T)}{8\log(T)/v}} \geq 
\bar{V}(n) - \frac{\sqrt{v\bar{V}(n)}}{2} \geq \bar{V}(n) - \frac{\bar{V}(n)}{2}= \frac{\bar{V}(n)}{2}\geq \frac{v}{2}\;,\]
and hence 
$nL(n) \geq 64\log(T)/v \times v/2 >64\log(T)>16 \log T$.

\item Otherwise, $\bar{V}(n)<\min(v,1/2)$ and hence
\[
L(n) \geq \bar{V}(n) -\sqrt{\frac{2\bar{V}(n)\log(T)}{n}}  
 \geq \bar{V}(n) -\sqrt{\frac{2v\log(T)}{8\log(T)/v}}  = \bar{V}(n) - \frac{v}{2},
\]
and hence $L(n) \leq 16 \log(T)/n$ implies
\[ \bar{V}(n) \leq \frac{v}{2} + \frac{16\log(T)}{n} = \frac{v}{2} + \frac{16\log(T)}{64\log(T)/v} = \frac{3v}{4}\;.\]
\end{itemize}
To summarize: $\Po_v\big(n\,L(n) \leq 16 \log(T)\big) \leq \Po_v\big(\bar{V}(n) \leq 3v/4)\leq \exp\big(-n\;\mathrm{kl}(3v/4, v)\big)$.

Lemma \ref{lem:kl} yields for $\alpha = -1/4$
\[\Po_v\left( \bar{V}(n)\leq \frac{3v}{4}\right) \leq \exp\left(-\frac{3vn}{88}\right) < \exp\left(-\frac{vn}{30}\right)  \;.  \]

To conclude, we bound the expectation of $\rho_1 $ as :
\begin{align*}
\E_v[\rho_1] &\leq \sum_{n=1}^\infty \Po_v(\rho_1\geq n)\\
& \leq n_0 + \sum_{n=n_0+1}^\infty \Po_v(\rho_1 \geq n)\\
& = \frac{64\log(T)}{v} + \sum_{n=n_0}^\infty \exp\left(-\frac{vn}{30}\right) +1 \\
&\leq  \frac{64\log(T)}{v} + \exp\left(-\frac{vn_0}{30}\right) \sum_{n=n_0}^\infty \exp\left(-\frac{vn}{30}\right)+ 1\\
&  \leq  \frac{64\log(T)}{v} + \exp\left(-\ceil[\bigg]{\frac{64\log(T)}{30}}\right) \frac{1}{1-\exp(-v/30)}+ 1\\
&\leq   \frac{64\log(T) + 60T^{-1/2}}{v} +1\;.
\end{align*}
since $1/(1-\exp(-u)) \leq 2/u$ for $0\leq u \leq 1$.

\end{proof}

The following lemma shows that the expectation of $\rho_0$ is small when $v \leq \frac{1}{3}T^{\frac{1}{3}}$.
\begin{lemma}\label{lem:exp_tau_0}
If $v\leq \frac{1}{3}T^{\frac{1}{3}}$, the expectation of the stopping time $\rho_0$ is bounded by:
$$\E_v[\rho_0] \leq 32  T^{\frac{1}{3}} \log T+ 1.$$
\end{lemma}
\begin{proof}
We use the following classical equality
\begin{align*}
\E_v[\rho_0]= \sum_{n=0}^T \Po_v(\rho_0>n).
\end{align*}
We can bound the probabilities in the latter sum by:
\begin{align*}
\Po_v (\rho_0>n) &\leq \Po_v \left(U(n) > \frac{1}{T^{\frac{1}{3}}}\right)\\
&\leq \Po_v \left(\bar{V}(n)>\frac{1}{2 T^{\frac{1}{3}}}\right) + \Po_v \left(U(n) - \bar{V}(n)>\frac{1}{2 T^{\frac{1}{3}}}, \bar{V}(n)<\frac{1}{2 T^{\frac{1}{3}}}\right).
\end{align*}

By Lemma \ref{lem:gen_Pinsker_strong},  $kl(\bar{V}(n), U(n))= \frac{2\log T}{n}\geq \frac{3(\bar{V}(n)- U(n))^2}{2(2U(n)+ 3\bar{V}(n))}\geq \frac{3(\bar{V}(n)- U(n))^2}{2\left(\frac{3}{2 T^{\frac{1}{3}}} + 2(U(n) -\bar{V}(n))\right)}$, when $\bar{V}(n)<\frac{1}{2 T^{\frac{1}{3}}}$. \\
This implies that $\Delta(n) := U(n)-\bar{V}(n)$ satisfies
$$ \frac{4 \log T}{3n}\left(\frac{3}{2 T^{\frac{1}{3}}}+ 2\Delta(n) \right)- \Delta(n)^2\geq 0$$
which yields
\begin{align*}
\Delta(n) \leq \frac{4\log T}{3n}+ \sqrt{\frac{\log^2 T}{9n^2}+\frac{8\log T}{3n} \frac{3}{2 T^{\frac{1}{3}}}}.
\end{align*}

This means that, if $\frac{\log T}{n}\leq \frac{1}{32 T^{\frac{1}{3}}}$, 
\begin{align*}
\Delta(n) \leq \frac{1}{2T^{\frac{1}{3}}}
\end{align*}

Hence, if $n > 32 T^{\frac{1}{3}}\log T$,

\begin{align}
\Po_v(\rho_0>n) 
&\leq \Po_v\left(\bar{V}(n)>\frac{1}{2 T^{\frac{1}{3}}}\right)\notag\\ 
&\leq  \exp\left(-kl\left(v, \frac{1}{2 T^{\frac{1}{3}}}\right)n\right)\notag\\
&\leq  \exp\left(-kl\left(\frac{1}{3 T^{\frac{1}{3}}}, \frac{1}{2 T^{\frac{1}{3}}}\right)n\right) \notag\\
&\leq  \exp \left(-\frac{1}{32 T^{\frac{1}{3}}}n\right), \label{eq:32T}
\end{align}
where we used the second inequality of Lemma \ref{lem:kl}, in the last inequality.\\
We conclude that $\E_v[\rho_0]< 32  T^{\frac{1}{3}} \log T+ 64 T^{-32/32} T^{\frac{1}{3}}<32  T^{\frac{1}{3}}\log T+ 64 T^{-\frac{2}{3}}$, since $1/(1-\exp(-u)) \leq 2/u$ for $0\leq u \leq 1$.
 \end{proof}
 
 In order to show that the regret of ETGstop is bounded by $O(T^{\frac{1}{3}})$ when $v$ is small, we prove that the probability of $\rho_1<\rho_0$ is small, for $v$ small enough.
  \begin{lemma}\label{lem:comp_tau}
 If $v< \frac{1}{2}T^{\frac{1}{3}}$,
$$ \Po_v(\rho_1<\rho_0)\leq 64 T^{-\frac{2}{3}}+ 32 T^{\frac{1}{3}}\log T \frac{1}{T^2}$$
\end{lemma}
\begin{proof}
Set $v< \frac{1}{2}T^{\frac{1}{3}}$.\\
Recall that $\Po_v(L(n)>v) \leq \frac{1}{T^2}$.
Hence if $n\leq 32 T^{\frac{1}{3}} \log T $, then  $\frac{16 \log T}{n}>\frac{1}{2 T^{\frac{1}{3}}}> v$, and $\Po_v(nL(n)>16 \log T)\leq \frac{1}{T^2}$.\\
Therefore $\Po_v(\rho_1< 32 T^{\frac{1}{3}}\log T, \rho_1 <\rho_0) \leq 32 T^{\frac{1}{3}}\log T \frac{1}{T^2}$.
Moreover, 
\begin{align*}
\Po_v\left(\rho_1> 32 T^{\frac{1}{3}}\log T, \rho_1 <\rho_0\right)
&\leq \Po_v \left(\exists n, n > 32 T^{\frac{1}{3}}\log T, U(n)>T^{-\frac{1}{3}} \right)\\
&\leq \sum_{n=32 T^{\frac{1}{3}}\log T}^T \Po_v\left(U(n)>T^{-\frac{1}{3}}\right)\\
&\leq \sum_{n=32 T^{\frac{1}{3}}\log T}^T  \exp\left(-\frac{1}{32}T^{-\frac{1}{3}}\right)\\
&\leq 64 T^{\frac{1}{3}}\frac{1}{T^{32/32}}\\
&\leq 64 T^{-\frac{2}{3}},
\end{align*}
where we used Equation \eqref{eq:32T} in the third  inequality, and $1/(1-\exp(-u)) \leq 2/u$ for $0\leq u \leq 1$ in the fourth inequality.
\end{proof}
 \subsection{Proof of Theorem \ref{th:ETGstop}}
 We separately study the cases  where $v \leq  \frac{1}{3 T^{\frac{1}{3}}} $, $\frac{1}{3 T^{\frac{1}{3}}}\leq v \leq \frac{1}{T^{\frac{1}{3}}}$ and $v \geq \frac{1}{T^{\frac{1}{3}}} $, and prove that in each case, $R_T(v)\leq O(T^{\frac{1}{3}}\log^2 T)$. 
 For $v \geq \frac{1}{T^{\frac{1}{3}}} $, we also prove $$R_T(v)\leq 7 + \frac{64\log(T) + 60T^{-1/2}}{v}+ \frac{4}{F(v/2)} + \beta \frac{ \log^2 T}{F(v/2)}.$$
 \subsubsection{Case when $v \leq  \frac{1}{3 T^{\frac{1}{3}}}. $}
When $v \leq  \frac{1}{3 T^{\frac{1}{3}}}$, we bound the regret as follows
\begin{align*}
R_T(v)&\leq T \Po_v(\rho_1<\rho_0) + \E_v[\rho_0] + \E_v\left[ \1(\rho_0<\rho_1)\sum_{t=\rho_0}^T r_t\right]
\\&\leq T \Po_v(\rho_1<\rho_0) + \E_v[\rho_0] + T \times (v)^2 
\\&\leq T\left(64 T^{-\frac{2}{3}}+ 32 T^{\frac{1}{3}}\log T \frac{1}{T^2} \right)  + 32 T^{\frac{1}{3}}\log T  + 64  T^{-\frac{2}{3}} + \frac{1}{9} T^{\frac{1}{3}} \\
& \leq  64 + 65  T^{\frac{1}{3}} +  33 T^{\frac{1}{3}} \log T,
\end{align*}
thanks to Lemmas \ref{lem:comp_tau} and \ref{lem:exp_tau_0}.
\subsubsection{Case when $\frac{1}{3 T^{\frac{1}{3}}}\leq v \leq \frac{1}{T^{\frac{1}{3}}}. $}
\begin{align*}
R_T(v) &\leq
\E_v\left[\1(\rho_1< \rho_0)\sum_{t=1}^T r_t\right]+ \E_v\left[\1(\rho_0< \rho_1)\sum_{t=1}^T r_t\right]\\
&\leq \E_v[\rho_0 \1(\rho_0<\rho_1)] + \E_v[\rho_1 \1(\rho_1<\rho_0)]\\
&~~ +\E_v \left[\sum_{t= {\rho_1}}^T r_t \1(\rho_1< \rho_0)\right] + \E_v\left[\sum_{t= {\rho_0}}^T r_t \1(\rho_0< \rho_1) \right]\\
&\leq \E_v[\rho_1] +
2 \max \left(\E_v\left[\sum_{t= {\rho_1}}^T r_t \1(\rho_1< \rho_0) \right],\E_v\left[\sum_{t= {\rho_0}}^T r_t \1(\rho_0< \rho_1) \right]\right).
\end{align*}
The expected regret in the second phase when $\rho_0$ has been reached first is bounded by 
$$\E_v\left[\sum_{t= {\rho_0}}^T r_t 1(\rho_0< \rho_1)\right]\leq \beta T \times T^{-\frac{2}{3}}\leq \beta T^{\frac{1}{3}},$$ while  the expected regret in the second phase when $\rho_1$ has been reached first is bounded by :
\begin{align*}
\E_v\left[\sum_{t= {\rho_0}}^T r_t \1(\rho_0< \rho_1)\right]&\leq 
\E_v \left[ \sum_{t=\rho_1}^T r_t \1\left(\cap_{s=\rho_1}^T \left\{\bar{V}(s)\geq \frac{v}{2}\right\}\right)\1( \rho_1< \rho_0) \right]  \\
&~~ + \sum_{t=\rho_1}^T \Po_v \left( \cup_{s=\rho_1}^T \left\{\bar{V}(s)< \frac{v}{2}\right\}\right) 
\\
 & \leq  \frac{ 4 +\beta \log^2 T}{F(v/2)}+ 3
 \\&
 \leq  \frac{24}{\ubar{\beta}} T^{\frac{1}{3}} + \frac{6\beta}{\ubar{\beta}}T^{\frac{1}{3}}\log^2 T, 
\end{align*}
thanks to Lemmas \ref{lem:prob_unfav_event} and \ref{lem:reg_fav_event}.
Therefore, 
$R_T(v)\leq 186 + 192 T^{\frac{1}{3}} + \frac{48}{\ubar{\beta}} T^{\frac{1}{3}} + \frac{12\beta}{\ubar{\beta}}T^{\frac{1}{3}}\log^2 T,$
according to Lemma \ref{lem:exp_tau}.

\subsubsection{Case when $v \geq \frac{1}{T^{\frac{1}{3}}}. $}
\begin{align*}
R_T(v) &\leq
\E_v\left[\1(\rho_0< \rho_1)\sum_{t=1}^T r_t \right]+ \E_v\left[\1(\rho_1< \rho_0)\sum_{t=1}^T r_t\right]\\
&\leq T \Po(\rho_0< \rho_1)+ \E_v\left[\sum_{t=1}^T r_t\1(\rho_1< \rho_0)\right]\\
&\leq 1 + \E_v[\rho_1] + \E_v \left[ \sum_{t=\rho_1}^T r_t \1\left(\cap_{s=\rho_1}^T \left\{\bar{V}(s)\geq \frac{v}{2}\right\}\right)\1( \rho_1< \rho_0) \right]  \\
& +\sum_{t=\rho_1}^T \Po \left( \cup_{s=\rho_1}^T \left\{\bar{V}(s)< \frac{v}{2}\right\}\right) .
\end{align*}
We conclude that
$$R_T\leq 7 + \frac{64\log(T) + 60T^{-1/2}}{v}+ \frac{4}{F(v/2)} + \beta \frac{ \log^2 T}{F(v/2)},$$
thanks to Lemmas \ref{lem:exp_tau}, \ref{lem:reg_fav_event} and \ref{lem:prob_unfav_event}. In particular : 
$$R_T\leq 7 + 60 T^{\frac{1}{3}} + \frac{8}{\ubar{\beta}} T^{\frac{1}{3}} + \frac{2\beta}{\ubar{\beta}}T^{\frac{1}{3}}\log^2 T.$$
\section{Proof of Theorem \ref{th:lower_bound_ETG}}\label{pr:minimax_lower}
We restate the theorem for the sake of readability.
\begin{repeatthm}{th:lower_bound_ETG}
If F admits a density lower bounded by $\ubar{\beta}>0$, the regret of any ETG strategy satisfies
$$\sup_{v\in[0,1]}R_T(v)\geq \frac{\ubar{\beta}}{4} \left(T^{\frac{1}{3}}-1 \right).$$
\end{repeatthm}

\begin{proof}
Let $v= a{T}^{-\frac{1}{3}}$.
ETG strategies are fully characterized by their choice of $\tau_0$ and $\tau_1$. Let us fix an ETG strategy, and let $\tau_m = \min(\tau_0, \tau_1)$.\\
Either $\Po_v[\tau_m>T^{\frac{1}{3}}]>1/2$ and hence $\E[\tau_m]\geq \frac{T^{\frac{1}{3}}}{2}$, which yields $R_T(v)\geq \ubar{\beta} (1- aT^{-\frac{1}{3}})^2\frac{T^{\frac{1}{3}}}{2}\geq \ubar{\beta} \frac{T^{\frac{1}{3}}-2a}{2}$, or  $\Po_v(\tau_m< T^{\frac{1}{3}})\geq \frac{1}{2}$, and in this case, 
\begin{align*}
\Po_v(\bar{V}_{\tau_m}=0, \tau_m< T^{\frac{1}{3}})
&\geq \sum_{k=1}^{T^{\frac{1}{3}}} \Po_v(\bar{V}_{\tau_m}=0|\tau_m=k) \Po_v(\tau_m = k)\\
&\geq \sum_{k=1}^{T^{\frac{1}{3}}} (1- a{T}^{-\frac{1}{3}})^k \Po_v(\tau_m = k)\\
&\geq \sum_{k=1}^{T^{\frac{1}{3}}} \exp(-a) \Po_v(\tau_m = k)\\
&\geq \frac{1}{2}\exp(-a),
\end{align*}
where the third inequality comes from the fact that $T^{\frac{1}{3}}\log(1 -a T^{-\frac{1}{3}}) \geq -a$.

If $\bar{V}_{\tau_m}=0$, whatever the order in which $\tau_0$ and $\tau_1$ are reached, every bid in the second phase will be zero. Therefore, 
$R_T(v)\geq \ubar{\beta}(T-T^{\frac{1}{3}}) a^2T^{-\frac{2}{3}}\geq \ubar{\beta}a^2( T^{\frac{1}{3}}-1)$. Hence
if $\Po_v(\tau< T^{\frac{1}{3}})\geq \frac{1}{2}$, $$R_T(v)\geq \frac{\ubar{\beta}}{2}\exp(-a) a^2\left(T^{\frac{1}{3}}-1 \right).$$

To conclude, we pick $a=2$ to obtain

$$\sup_{v\in[0,1]}R_T(v)\geq \min\left(
\frac{\ubar{\beta}}{2}\times4\exp(-2) \left(T^{\frac{1}{3}}-1 \right), \frac{\ubar{\beta}}{2} T^{\frac{1}{3}}-2  \right) \geq \frac{\ubar{\beta}}{4} \left(T^{\frac{1}{3}}-1 \right).$$
\end{proof}

\end{document}